\documentclass[twoside]{article}

\usepackage[accepted]{aistats2020}
\usepackage{enumerate}
\usepackage[utf8]{inputenc} 
\usepackage[T1]{fontenc}    
\usepackage{hyperref}       
\usepackage{url}            
\usepackage{booktabs}       
\usepackage{amsfonts}       
\usepackage{nicefrac}       
\usepackage{microtype}      
\usepackage{bbm}
\usepackage{mathtools}
\usepackage{amsthm}
\usepackage[usenames, dvipsnames]{color}
\usepackage{bm}
\usepackage{url}
\usepackage{caption}
\usepackage{subcaption}
\usepackage{xcolor}
\usepackage[T1]{fontenc}
\usepackage{lmodern}
\usepackage{float}
\usepackage{natbib}
\usepackage{dblfloatfix}

\makeatletter
\newcommand{\pushright}[1]{\ifmeasuring@#1\else\omit\hfill$\displaystyle\left(#1\right)$\fi\ignorespaces}
\newcommand{\pushleft}[1]{\ifmeasuring@#1\else\omit$\displaystyle#1$\hfill\fi\ignorespaces}

\newtheorem{corollary}{Corollary}[section]

\allowdisplaybreaks

\newtheorem{theorem}{Theorem}
\newtheorem{lemma}{Lemma}

\def \c {c} 
\def \cc {\mathbf{c}} 	
\def \L {\lambda}		
\def \b {b}			
\def \B {\boldsymbol{b}}			

\def \ttheta {\boldsymbol{\theta}}	

\newcommand{\given}{\,|\,}

\def \U {U}

\newcommand{\W}{\boldsymbol{w}}
\newcommand{\x}{\mathbf{x}}

\newcommand{\y}{\mathbf{y}}

\def \E {\mathbb{E}}

\usepackage[nolist]{acronym}
\begin{acronym}
	\acro{PoRB-Net}{POisson process Radial Basis neural NETwork}
\end{acronym}
\begin{document}

%
\runningtitle{Poisson Process Radial Basis Function Networks}

%

\twocolumn[

\aistatstitle{Towards Expressive Priors for Bayesian Neural Networks:\\ Poisson Process Radial Basis Function Networks}

\aistatsauthor{ Beau Coker \And Melanie F. Pradier \And  Finale Doshi-Velez }

\aistatsaddress{ Biostatistics, Harvard University \And  SEAS, Harvard University \And SEAS, Harvard University }
]

\begin{abstract}
	While Bayesian neural networks have many appealing characteristics, current priors do not easily allow users to specify basic properties such as expected lengthscale or amplitude variance.  In this work, we introduce Poisson Process Radial Basis Function Networks, a novel prior that is able to encode amplitude stationarity and input-dependent lengthscale. We prove that our novel formulation allows for a decoupled specification of these properties, and that the estimated regression function is consistent as the number of observations tends to infinity. We  demonstrate its behavior on synthetic and real examples.
\end{abstract}

\vspace{-0.5cm}
\section{Introduction}

Neural networks (NNs) are flexible universal function approximators that have been applied with success in many domains~\citep{lecun_deep_2015}.  When data are limited, Bayesian neural networks (BNNs) capture function space uncertainty in a principled way~\citep{hinton_bayesian_1995} by placing priors over network parameters. 
Unfortunately, priors in parameter space often lead to unexpected behavior in function space; standard BNN priors do not even allow us to encode basic properties such as stationarity, lengthscale, or
signal variance~\citep{lee_bayesian_2004}.

This is in contrast to Gaussian processes (GPs), which can easily encode these properties via the covariance function. Still there are many situations in which we may prefer to use BNNs rather than GPs: BNNs may be computationally more scalable, especially at test time, and having an explicit parametric expression for posterior samples is convenient when additional computation is needed on the function, such as Thompson sampling~\citep{thompson_likelihood_1933} or predictive entropy search~\citep{hernandez-lobato_predictive_2014}.

Therefore, a natural question arises: can we design BNN priors that encode stationarity properties like a GP while retaining the benefits of BNNs?  Recent works have started toward the path of creating BNNs with certain functional properties. Some approaches use sample-based methods to evaluate the mis-match between the distribution over functions and a reference distribution with desired properties~\citep{flam-shepherd_mapping_2017, sun_functional_2019}.  Another approach explores different BNN architectures to recover equivalent GP kernel combinations in the infinite width limit~\citep{pearce_expressive_2019}. In contrast, we directly incorporate functional properties via an alternative parametrization and well-designed prior over the network weights, without sample-based optimizations nor infinite width-limit assumptions. 

In this work, we introduce Poisson-process Radial Basis function Networks (PoRB-NETs), a Bayesian formulation which places a Poisson process prior~\citep{kingman_poisson_1992} over the center parameters in a single-layer Radial Basis Function Network
(RBFN)~\citep{lippmann_pattern_1989}.
The proposed formulation enables \textit{direct} specification of  stationary amplitude variance and (non)-stationary lengthscale.  When the input-dependence of the lengthscale is unknown, we derive how it can be inferred.
An important technical contribution is that PoRB-NETs ensure that amplitude variance and lengthscale are \emph{decoupled}, that is, each can be specified independently of the other (which does not occur in a naive application of a Poisson process to determine neural network centers).  As with GPs, and unlike networks that force a specific property~\citep{anil_sorting_2018}, these prior properties will also adjust given data.

Specifically, we make the following contributions: (i) we introduce a novel, intuitive prior formulation for RBFNs that encodes a priori distributional knowledge in function space, decoupling notions of lengthscale and signal variance in the same way as a GP; (ii) we prove important theoretical properties such as consistency or stationarity; (iii) we provide an inference algorithm to learn input-dependent lengthscale for a decoupled prior; and (iv) we empirically demonstrate the potential of our approach on synthetic and real examples.


\section{Related Work}\label{sec:related}

\vspace{-0.1cm}

\paragraph{Early weight-space priors for BNNs.} Most classic Bayesian formulations of NNs use priors for regularization and model selection while minimizing the amount of undesired functional prior~\citep{lee_bayesian_2004,  muller_issues_1998} (because of the lack of interpretability in the parameters).
\cite{mackay_bayesian_1992} proposes a hierarchical prior\footnote{Hierarchical priors are convenient when there is a lack of interpretability in the parameters. As the addition of upper levels to the prior reduces the influence	of the particular choice made at the top level, the resulting prior at the bottom level (the  original parameters) will be more diffuse~\citep{lee_bayesian_2004}} combined with empirical Bayes.
\cite{lee_noninformative_2003} proposes an improper prior for neural networks, which avoids the injection of artifact prior biases at the cost of higher sensitivity to overfitting.
\cite{robinson_priors_2001} proposes priors to alleviate overparametrization of NN models.
Instead, we focus on obtaining certain functional properties through prior specification. 


\vspace{-0.1cm}
\paragraph{Function-space priors for BNNs.}
%
Works such as \citep{flam-shepherd_mapping_2017, sun_functional_2019} match BNN priors to specific GP priors or functional priors via sampled-based approximations which rely on sampling function values at a collection of input points $x$.  These approaches do not provide guarantees outside of the sampled region, and even in the sampled region, their enforcement of properties will be approximate.  Neural processes~\citep{garnelo_neural_2018} use meta-learning to identify functional properties that may be present in new functions; they rely on having many prior examples and do not allow the user to specify basic properties directly.   
%
%
In contrast, our approach encodes functional properties directly through prior design, without relying on function samples.

\vspace{-0.1cm}
\paragraph{Bayesian formulations of RBFN models.}
Closest to our work are Bayesian formulations of RBFNs.  \cite{barber_radial_1998} and \cite{yang_bayesian_2005} consider Bayesian RBFNs for a fixed number of hidden units; the former propose Gaussian approximations to the posterior distribution which, for fixed basis function widths, is analytic in the parameters. In contrast, our work infers the number of hidden units from data.
\cite{holmes_bayesian_1998} and \cite{andrieu_robust_2001} propose full Bayesian formulations based on Poisson processes (and a reversible jump Markov chain Monte Carlo inference scheme).  Their focus is on learning the number of hidden nodes, but they neither prove functional or inferential properties of these models nor allow for the easy incorporation of functional properties.  

\section{Background}\label{sec:rbfn}

\vspace{-0.1cm}
\paragraph{Bayesian neural networks (BNNs).}
We consider regressors of the form $y =f_{\W,\B}(x) + \epsilon$, where
$\epsilon$ is a noise variable and $\W$ and $\B$ refers to the weights and biases of a
neural network, respectively. In the
Bayesian setting, we assume some prior over the weights and biases $\W,\B \sim
p(\W,\B)$. One common choice is to posit i.i.d. normal priors over each
network parameter $\W \sim \mathcal{N} \left(0,\sigma^2_w \mathbf{I} \right)$ and $\B \sim \mathcal{N} \left(0,\sigma^2_b \mathbf{I} \right)$: we will refer to such model as standard BNN~\citep{neal_priors_1996}.

\vspace{-0.1cm}
\paragraph{Radial basis function networks (RBFNs).}
Radial Basis Function Networks (RBFNs) are classical shallow neural networks that approximate arbitrary nonlinear functions through a linear combination of radial kernels~\citep{powell_algorithms_1987, gyorfi_distribution-free_2002}.  They are universal function approximators~\citep{park_universal_1991} and are widely used across a wide variety of disciplines including numerical analysis, biology, finance, and classification in spatio-temporal models~\citep{dash_radial_2016}.
For input $x\in \mathbb{R}^D$, the output of a single-hidden-layer RBFN of width $K$ is given by:
\vspace{-0.2cm}
\begin{equation}
f(x \given \ttheta)  =  \b + \sum^K_{k=1} w_k \exp( - s_k^2 (x-c_k)^T (x-c_k)), \label{eq:RBFN}
\end{equation}
where $s_k^2 \in \mathbb{R}$ and $c_k \in \mathbb{R}^D$ are the scale and center parameters, respectively, $w_k \in \mathbb{R}$ are the hidden-to-output weights, and $\b \in \mathbb{R}$ is the bias parameter. Each $k$-th hidden unit can be interpreted as a local receptor centered at $c_k$, with radius of influence $s_k$, and relative importance $w_k$~\cite{gyorfi_distribution-free_2002}.

\paragraph{Poisson process.}
A Poisson process~\citep{kingman_poisson_1992} on $\mathbb{R}^D$ is defined by a positive real-valued intensity function $\lambda(c)$. For any set $\mathcal{C}\subset \mathbb{R}^D$, the number of points in $\mathcal{C}$ follows a Poisson distribution with parameter $\int_\mathcal{C} \lambda(c) dc$. The process is called \textit{inhomogenous} if $\lambda(c)$ is not a constant function. We use a Poisson process as a prior on the center parameters of a radial basis function network. 

\paragraph{Gaussian Cox process.} A Bayesian model consisting of a Poisson process likelihood and a Gaussian process prior $g(c)$ on the intensity function $\lambda(c)$ is called a Gaussian Cox Process~\citep{moller_log_1998}. \cite{adams_tractable_2009} present an extension to this model, called the \textit{Sigmoidal Gaussian Cox Process}, that passes the Gaussian process through a scaled sigmoid function.  We will use this process to learn input-dependent wiggliness (lengthscale) of the function by adapting the number of hidden units in the RBFN network.  

\section{Poisson-process radial basis function networks (PoRB-NET)}\label{sec:PoRB-NET}

In this section, we introduce Poisson-process Radial Basis function Networks (PoRB-NETs). This model achieves two essential desiderata for a NN prior.  First, it enables the user to encode the fundamental basic properties of smoothness (i.e., lengthscale), amplitude (i.e., signal variance), and (non)-stationarity.  Second, PoRB-NETs adapt the complexity of the network based on the inputs.  For example, if the data suggests that the function needs to be less smooth in a certain input region, then that data will override the prior. Importantly, PoRB-NET fulfills these desiderata while retaining appealing computational properties of NN-based models such as fast computation at test time and explicit, parametric samples.

\paragraph{Generative model.}
As in a standard BNN, we assume a Gaussian likelihood centered on the network output and independent Gaussian priors on the weight and bias parameters.
Unique to the novel PoRB-NET formulation is a Poisson process prior over the set of center parameters and a deterministic dependence of the scale parameters on the Poisson process intensity:
\begin{eqnarray}
\cc \given \L & \sim & \exp \left( - \int_\mathcal{C} \L(\cc)d\cc \right) \prod^K_{k=1} \L(\c_k) \\
s_k^2 \given \L, \cc & = & s_0^2 \L^2(c_k) \\
w_k & \sim & \mathcal{N} \left(0,\sqrt{s_0^2/\pi}\sigma^2_w I \right) \\
\b & \sim & \mathcal{N}(0,\tilde{\sigma}^2_b ) \\
y_n \given x_n,\ttheta & \sim & \mathcal{N}(f(x_n;\ttheta),\sigma^2_x),\label{eq:generative}
\end{eqnarray}
where $f(x_n;\ttheta)$ is given by Eq.~\eqref{eq:RBFN}, $\ttheta$ denotes the set of RBFN parameters including the centers, weights and bias; $\L: \mathcal{C}\to\mathbb{R}^+$ is the (possibly) inhomogeneous Poisson process intensity; and $s_0^2$ is a hyperparameter that defines the scale of the RBF basis function when the intensity is one.
%

Different priors could be considered for the intensity function $\L$. The simplest case is to assume a constant intensity defined over a bounded region $\mathcal{C}$, i.e., with an abuse of notation $\L(c) = \L$ and $\L^2 \sim \mathrm{Gamma}(\alpha_\lambda,\beta_\lambda)$.
Under this specific formulation, we prove in Section ~\ref{sec:properties} that the posterior regression function is consistent as the number of observations tends to infinity and we prove in Section~\ref{sec:covariance} that the amplitude variance is stationary as the size of the region $\mathcal{C}$ tends to infinity; such amplitude variance only depends on the variance of the hidden-to-output weights and output bias $\mathbb{V}[f(x)] \approx \sigma^2_b + \sigma^2_w$. We further show that the intensity $\L$ controls the lengthscale.

\paragraph{Hierarchical prior for unknown lengthscale.}
In the case when the input-dependence of the lengthscale is unknown, we further model the intensity function $\L(c)$ of the Poisson process by a sigmoidal Gaussian Cox process~\cite{adams_tractable_2009}:
\begin{eqnarray}
h & \sim & \mathrm{GP}(0, C(\cdot,\cdot)) \\
\lambda^* & \sim & \text{Gamma}(\alpha_\lambda, \beta_\lambda) \\
\L(c) & = & \lambda^* \text{sigmoid}(h(c)),
\end{eqnarray}
where $\lambda^*$ is an upper bound on the intensity.  In Section~\ref{sec:discussion}, we discuss alternative link functions and why the PoRB-NET formulation is a natural way to satisfy our desiderata.

\paragraph{Contrast to BNNs with i.i.d. Gaussian weight-space priors.}
In Sections~\ref{sec:properties} and~\ref{sec:covariance}, we prove that the proposed formulation has the desired properties described above.
However, before doing so, we briefly emphasize that the i.i.d. Gaussian weight-space prior that is commonly used with BNNs does not enjoy these properties.
To see why, let us consider a standard feed-forward NN layer with a $D=1$ dimensional input and an RBF activation function.  We can rewrite the hidden units as $\sigma(w_k x + b_k) = \sigma(w_k (x - ( -b_k/ w_k)))$. This means that the corresponding center of the $k$-th hidden unit is $c_k = -b_k/ w_k$. If $b_k$ and $w_k$ are assigned independent Gaussian priors with zero mean, as is standard in a BNN, then the center parameter has a zero-mean Cauchy distribution.%
\footnote{If $b_k$ and $w_k$ have non-zero means, the ratio distribution for the corresponding center parameter $c_k$ has undefined mean and non-closed form median~\citep{cedilnik_distribution_2004}.}
This is a critical observation that motivates our work: \textit{A standard BNN concentrates the center of hidden units near the origin, resulting in nonstationary priors in function space.}

\begin{figure}[h]
	\includegraphics[width=0.5\textwidth]{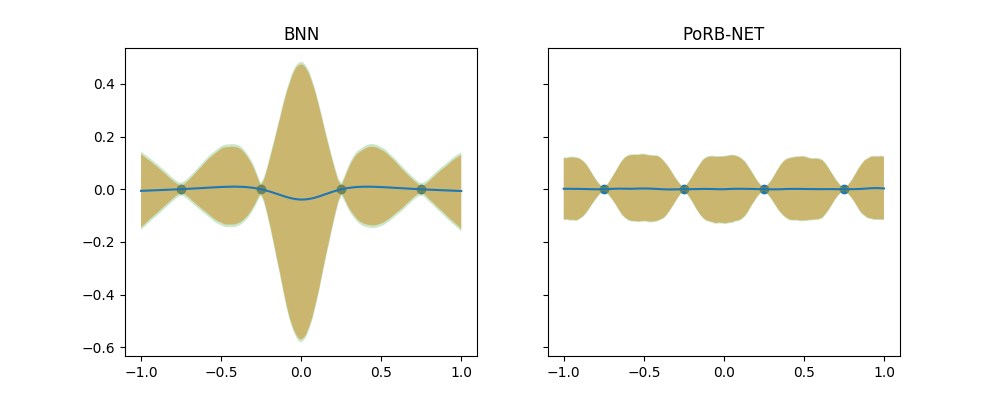}
	\caption{\textbf{PoRB-NET can capture amplitude stationarity.} Standard BNN priors suffer from amplitude-nonstationarity while PoRB-NETs fix this issue. We show the posterior predictive for both models on a simple 1-dimensional problem.}\label{fig:stationarity}
\end{figure}

\vspace{-0.2cm}
\section{Consistency of PoRB-NET predictions}
\vspace{-0.2cm}
\label{sec:properties}
In this section, we study \emph{consistency of predictions}.  That is, as the number of observations goes to infinity, the posterior predictive should concentrate around the correct function.  When dealing with priors that can produce an unbounded number of parameters, consistency is a basic but important property.  To our knowledge, we are the first to provide consistency proofs for Poisson process networks (no consistency guarantees were derived by~\citet{andrieu_robust_2001}).  

Formally, let $g_0(x)$ be the true regression function, and $\hat{g}_n(x)$ the estimated regression function $\hat{g}_n(x) = \mathbb{E}_{\hat{f}_n}[Y\mid X]$, where $\hat{f}_n$ is the estimated density in parameter-space based on $n$ observations.
The estimator $\hat{g}_n(x)$ is said to be consistent with respect to the true regression function $g_0(x)$ if, as $n$ tends to infinity:
\begin{equation}
\int (\hat{g}_n(x) - g_0(x))^2 ~dx \xrightarrow{p} 0.\label{eq:freqconsistency}
\end{equation}

Doob's theorem shows that Bayesian models are consistent as long as the prior places positive mass on the true parameter~\cite{doob_application_1949, miller_detailed_2018}. For finite dimensional parameter spaces, one can ensure consistency by simply restricting the set of zero prior probability to have small or zero measure.
%
Unfortunately, in infinite dimensional parameter spaces, this set might be very large~\cite{freedman_asymptotic_1963, wainwright_high-dimensional_2019}.
In our case where functions correspond to uncountably infinite sets of parameters, it is impossible to restrict this set of inconsistency to have measure zero.


Instead, in the following we aim to show a strong form of consistency called Hellinger consistency. We follow the approach of~\cite{lee_consistency_2000}, who shows consistency for regular BNNs with normal priors on the parameters. 
Formally, let $(x_1,y_1),\dotsc,(x_n,y_n) \sim f_0$ be the observed data drawn from the ground truth density $f_0$, and let us define the Hellinger distance between joint densities $f$ and $f_0$ over $(X,Y)$ as:
$$
D_H(f,f_0) = \sqrt{\int \int \left( \sqrt{f(x,y)} - \sqrt{f_0(x,y)}\right)^2  ~dx~dy}.
$$
The posterior is said to be consistent over Hellinger neighborhoods if for all $\epsilon>0$,
$$
p(\{f: D_H(f,f_0) \le \epsilon \})\xrightarrow{p} 1. 
$$
\cite{lee_consistency_2000} shows that Hellinger consistency of joint density functions implies frequentist consistency as described in Eq.~\eqref{eq:freqconsistency}.

\begin{theorem}\label{th:consistency} (Consistency of PoRB-NETs) 
	A radial basis function network with a homogeneous Poisson process prior on the location of hidden units is Hellinger consistent as the number of observations goes to infinity.
\end{theorem}
\begin{proof} 
	Leveraging the results from \cite{lee_consistency_2000}, we use bracketing entropy from empirical process theory to bound the posterior probability outside Hellinger heighborhoods. We need to check that our model satisfies two key conditions. Informally, the first condition is that the prior probability placed on parameters larger in absolute value than a bound $B_n$, where $B_n$ is allowed to grow with the data, is asymptotically bounded \textit{above} by an exponential term $\exp(-nr)$, for some $r>0$. The second condition is that the prior probability placed on KL neighborhoods of the ground truth density function $f_0$ is asymptotically bounded \textit{below} by an exponential term $\exp(-n\nu)$, for some $\nu>0$. The full proof can be found in the Appendix.
\end{proof}
Note that consistency of predictions does not imply concentration of the posterior in weight space~\citep{izmailov_subspace_2019}; this distinction happens because radial basis function networks, like other deep neural models, are not identifiable~\citep{watanabe_almost_2007}.

\section{Amplitude, Lengthscale, and Stationarity}
\label{sec:covariance}
We now return to the core desiderata: we wish to be able to specify priors about the function's lengthscale and amplitude variance in a decoupled fashion. We want the same ease in specification as with an RBF kernel for a GP, where one can specify a lengthcale and a scaling constant on the covariance independently. We also want to specify whether these properties are stationary. 

To do so, we first derive the covariance of the proposed PoRB-NET model.  We consider the covariance function between two inputs $x_1$ and $x_2$, which we use to illustrate the specific form of non-stationarity exhibited by our model. The full derivations supporting this section are available in the Appendix.

\cite{neal_priors_1996} showed that the covariance function for a single-layer BNN with a \textit{fixed} number of hidden units $\rho(x; \theta_1),\dotsc, \rho(x; \theta_K)$ and independent $\mathcal{N}(0,\sigma^2_w)$ and $\mathcal{N}(0,\sigma^2_b)$ priors on the hidden-to-output weights and output bias takes the following general form:
\begin{equation}
\text{Cov}(f(x_1), f(x_2)) = \sigma^2_b + \sigma^2_w K \E_\theta\left[\rho(x_1; \theta) \rho(x_2; \theta) \right].\nonumber
\end{equation}
In the same spirit, we derive that the covariance function for a BNN (including RBFN) with a \textit{distribution} over the number of hidden units takes an analogous form, replacing the number of hidden units $K$ with the expected number of hidden units $\mathbb{E}[K]$:
\begin{equation*}
\resizebox{0.93\hsize}{!}{$%
	\text{Cov}(f(x_1), f(x_2)) = \sigma^2_b + \sigma^2_w \E\left[K \right] \underbrace{\E_\theta \left[\rho(x_1; \theta) \rho(x_2; \theta) \right]}_{:=\U(x_1,x_2 )}.
	$%
}%
\end{equation*}
%
In the PoRB-NET model, $\theta=\{s,c\}$,\footnote{We drop subscript $k$ for notation simplicity.} $\rho(x; \theta)=\phi(s(x-c))$ where $\phi(x) = \exp(-x^2)$, and $\E\left[K \right]=\int_\mathcal{C} \lambda(c) ~dc$. The key term is the function $U(x_1,x_2)$, which depends on the functional form of the hidden units and the prior over its parameters $\theta$.  In the Appendix, we derive the general expression of $\U(x_1,x_2)$ for a non-homogeneous Poisson process, which has a closed-form expression in the case of a homogeneous Poisson process.  Below, we describe how, in the case of a homogeneous Poisson process prior, this form corresponds to an asymptotically stationary covariance when the bounded region $\mathcal{C}$ increases to infinity.
Finally, we prove that the form of the generative model proposed in Section~\ref{sec:PoRB-NET} separates notions of amplitude and lengthscale for any arbitrary choice of intensity function.

\paragraph{A homogeneous PP yields stationarity.}
In the case of constant intensity $\lambda(c)=\lambda$ defined over $\mathcal{C} = [C_0,C_1]$, the expression of $U(x_1,x_2)$ can be derived in closed form:
\begin{align}
& U (x_1,x_2) = \frac{\lambda}{\Lambda} \sqrt{\frac{\pi}{s^2}}
\exp\left\{-s^2 \left(\frac{x_1-x_2}{2} \right)^2\right\} \nonumber \\ &\left[\Phi((C_1-x_m) \sqrt{2s^2}) - \Phi((C_0-x_m) \sqrt{2s^2}\lambda)\right], \label{eq:cov_actual}
\end{align}
where $\Phi$ is the cumulative distribution function of a standard Gaussian and $x_m = (x_1+x_2)/2$ is the midpoint of the inputs.
As the bounded region $\mathcal{C}$  increases, the second term approaches one, and so the covariance of a PoRB-NET approaches a squared exponential kernel with inverse length-scale $s^2$ and amplitude variance $\sigma^2_w \lambda \sqrt{\pi/s^2}$:
\begin{align}
\text{Cov} &\left(f(x_1), f(x_2) \right) \approx 
\nonumber \\
& \sigma^2_b +  \sigma^2_w \lambda \sqrt{\frac{\pi}{s^2}}
\exp\left\{-s^2 \left(\frac{x_1-x_2}{2} \right)^2\right\}, \label{eq:cov}
\end{align}
which is stationary since it only depends on the squared difference between $x_1$ and $x_2$.
Notice that this result does not rely on an infinite width limit of the network, but only on the Poisson process region $[C_0,C_1]$ being relatively large compared to the midpoint $x_m$. In practice, $[C_0,C_1]$ can be set larger than the range of observed $x$ values to achieve covariance stationarity over the input domain.  Figure~\ref{fig:covariogram} shows that over the region $[-5,5]$ the analytical covariance from Equation~\eqref{eq:cov_actual} is fairly constant with only slight drops near the boundaries. 
In contrast, the covariance function of a radial basis function network \emph{without a Poisson process prior} is not even approximately stationary. For a Gaussian prior on the centers $c_k\sim \mathcal{N}(0, \sigma^2_c)$ and a fixed scale $s^2 = 1/(2\sigma^2_s)$, \cite{williams_computing_1997} shows that:
\begin{equation}
\resizebox{0.98\hsize}{!}{$%
	\begin{split}
	\U(x_1,x_2) \propto
	\underbrace{\exp\left(-\frac{(x_1-x_2)^2}{2(2\sigma^2_s + \sigma^4_s/\sigma^2_c)} \right)}_{Stationary}
	\underbrace{\exp\left(-\frac{x_1^2+x_2^2}{2(2\sigma^2_c + \sigma^2_s)} \right)}_{Nonstationary}, \nonumber
	\end{split}
	$%
}%
\end{equation}
and Figure~\ref{fig:covariogram} shows that, unlike our approach, a standard prior is highly non-stationary.  

\begin{figure}[H]
	\centering
	\includegraphics[width=0.53\textwidth]{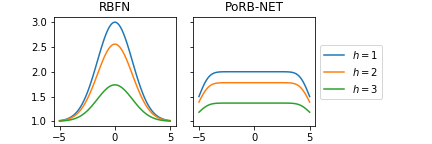}
	\caption{\textbf{PoRB-NET is able to express signal variance stationarity whereas RBFN cannot.} The plotted lines correspond to $\text{Cov}(x-h/2,x+h/2)$\label{fig:covariogram} for different gap values $h$. We set the weight variance to be 1, $s^2=1$, and the homogeneous PP to be defined over $\mathcal{C} = [-5,5]$ with 10 expected hidden units.}
\end{figure}


\paragraph{Decoupling of amplitude and lengthscale.}

For the Poisson process intensity $\lambda$ to solely play the role of lengthscale, the amplitude variance should not depend on $\lambda$. But for a fixed RBF scale parameter $s^2$, a higher intensity implies a higher number of basis functions, which implies a higher amplitude variance as the basis functions add up. Mathematically, the impact of the intensity can be seen in Equation (\ref{eq:cov_actual}). If we instead allow the scale parameters $s^2$ to increase as a function of the intensity, thus making the RBF basis functions more narrow, we can possibly counteract the increase in the number of basis functions. From the approximate covariance in Equation~\eqref{eq:cov}, we can see that setting $s^2=\lambda^2$ will achieve this goal. In our generative model, we set $s^2 = s^2_0 \lambda^2$ by introducing the hyperparameter $s^2_0$, which is useful in practice for adjusting the scale parameters independently of the network width.

By setting the RBF scale parameters as a function of the intensity, the approximate covariance in Equation~\ref{eq:cov} becomes:
\begin{align}
\text{Cov}\left(f(x_1), f(x_2) \right) &\approx 
\sigma^2_b + \tilde{\sigma}^2_w 
\exp\left\{-s^2_0 \lambda^2 \left(\frac{x_1-x_2}{2} \right)^2\right\}, \label{eq:cov2}
\end{align}
where we set $\tilde{\sigma}^2_w = \sqrt{\pi/s^2_0} \; \sigma^2_w$ for notational convenience. Therefore, not only does the intensity no longer impacts the amplitude variance $\tilde{\sigma}^2_w$ but it clearly plays the role of an inverse lengthscale in the squared exponential kernel.

\paragraph{A non-homogeneous PP yields non-stationarity.}

When the intensity is a non-constant function $\lambda(c)$, then the derivation yielding Equation~\eqref{eq:cov_actual} does not hold (see corresponding expression in the Appendix). However, we find that setting the scale parameter of each hidden unit based on the intensity evaluated at its center parameter yields approximate variance stationarity empirically, while allowing the intensity to control an input dependent lengthscale. That is, we set the scale of hidden unit $k$ to $s^2_k = s^2_0 \lambda(c_k)^2$.


\vspace{-0.1cm}
\section{Inference}\label{sec:inference}
\vspace{-0.2cm}

Our objective is to infer the posterior distribution over functions $p(f \given \y,\x)$, which is equivalent to inferring the posterior distribution over the weights  $p(\ttheta \given \y,\x)$.
Given the posterior distribution, we model predictions for new observations and their associated uncertainties  through the posterior predictive distribution:
\begin{eqnarray}
p(\y^{\star}|\x^{\star},\mathcal{D}) = \int p(\y^{\star}|\x^{\star},\W) p(\W | \mathcal{D}) d\W.
\end{eqnarray}
The PoRB-NET prior defines a function over an unbounded region.  In practice, because the radial basis function has a finite region of effect, we only need to perform posterior inference in regions near the data; in regions far from the data, we may simply sample centers from the prior. The size of the region to consider will depend on the widths of the radial basis functions.

Once this region is defined, we perform inference with a Markov-Chain Monte Carlo (MCMC) algorithm. Each iteration can be broken down into three steps. 
Step~1 updates all network parameters $\ttheta$, i.e. $\left(\{w_k, c_k\}_{k=1}^K, \b\right)$, conditional on the network width $K$ and intensity function $\L$. We update all network parameters at once with Hamiltonian Monte-Carlo (HMC) \citep{neal_priors_1996}. 
Step~2 updates the network width $K$ conditional on the network parameters and intensity function via birth and death Metropolis-Hastings steps.
Finally in Step~3, we compute a point-estimate of the Poisson process intensity by running an HMC subroutine and averaging multiple samples of the posterior distribution. 

\paragraph{Step 1: Update network parameters $\ttheta$.}
The full conditional distribution of the network parameters $\ttheta$ given the width of the network $K$ and intensity function $\L$ can be written as:
\vspace{-0.2cm}
\begin{align}
p(\ttheta \given \y, \x, K, \L) \propto 
& \left( \prod_{n=1}^N \mathcal{N}(y_n; f(x_n; \ttheta) \right) \nonumber
\mathcal{N}(b; 0, \sigma^2_{b}) \\
& \left( \prod_{k=1}^K \mathcal{N}(w_k; 0, \sigma^2_{w}) \; \L(c_k) \right).
\end{align}
We resort to HMC, which requires tuning $L$ leap-frog steps of size $\epsilon$, to propose updates from this full conditional distribution.
Notice that we do not have a parametric expression for the intensity $\L$, i.e., we cannot directly evaluate $\L(c_k)$; instead, we approximate this quantity by $\hat{\L}(c_k) = \mathbb{E}_{p(h|\y,\x,\theta)}\left[\lambda^{\star} \phi(h(c_k)) \right]$ in Step~3.

\paragraph{Step 2: Update network width $K$.} We adapt the network width via Metropolis-Hastings (MH) steps. We randomly sample a deletion or an insertion of a hidden unit (in practice, we do this multiple times per iteration) with equal probability. 
%

For a birth step, we sample the weight and scale parameters from their prior distributions and the center parameter from a Gaussian process conditioned on all of the center parameters, including auxiliary ``thinned'' center parameters introduced in Step 3.
In the case of a fixed intensity function, we propose the center parameter from $\lambda(c) / \int \lambda(c)~dc$.  For the death step, we propose to delete a hidden unit at random by uniformly selecting among the existing hidden units. 

Therefore, we can write the hidden unit deletion and insertion proposal densities as follows:
\begin{align*}
\begin{split}
q_{del}(K \to K-1) = &\frac{1}{2}\frac{1}{K} \\
q_{ins}(K \to K+1) = &\frac{1}{2} \frac{1}{\mu(\mathcal{T})} 
\mathcal{N}(w'; 0, \sigma^2_w)
\mathcal{N}(s'; 1, \sigma^2_s) \cdot \\
& \mathcal{G}\mathcal{P}(h'\mid \bm{h}_{M+K}, c', \{c_k\}, \{\tilde{c}_m\})
\end{split}
\end{align*}
and acceptance probabilities are:
\begin{align*}
a_{del} &= \frac{\prod_{n=1}^N \mathcal{N}(y_n; f(x_n; \{\ttheta_{i-1}\})}
{\prod_{n=1}^N \mathcal{N}(y_n; f(x_n; \{\ttheta_i\}) \lambda^* \sigma(h(c'))}
\frac{K}{\mu(T)}
\\
a_{ins} &= \frac{\prod_{n=1}^N \mathcal{N}(y_n; f(x_n; \{\ttheta_{i+1}\}) \lambda^*\sigma(h(c'))}
{\prod_{n=1}^N \mathcal{N}(y_n; f(x_n; \{\ttheta_{i}\})}
\frac{\mu(T)}{K+1}
\end{align*}
where $\{\ttheta_{K-1} \}$ indicates the parameters for a network with $K-1$ hidden units and $\{\ttheta_{K+1} \}$ indicates the parameters for a network with $K+1$ hidden units. 

\begin{figure*}
	\begin{center}
		\includegraphics[width=1.0\textwidth]{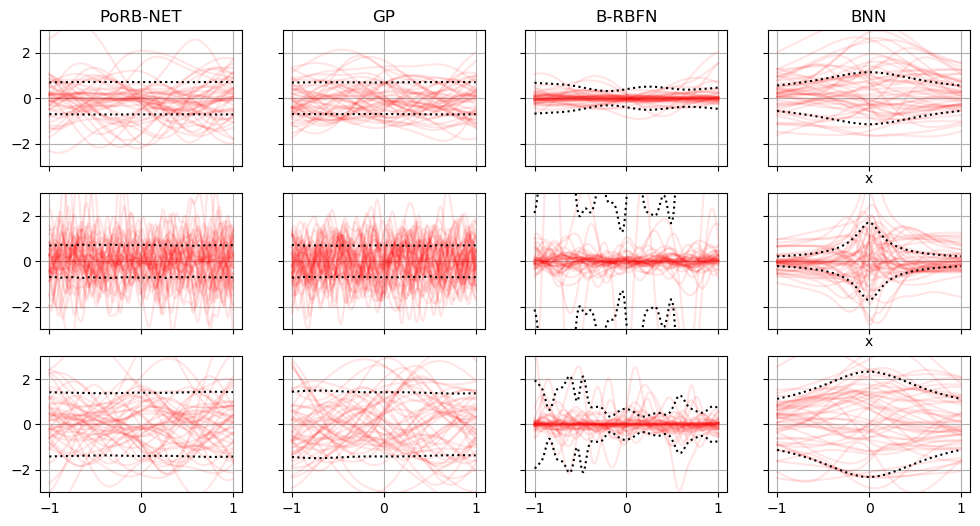}
		\caption{\textbf{PoRB-NET allows for easy specification of lengthscale and signal variance like a GP.} We show prior samples from PoRB-NET with homogeneous intensity, GP prior~\citep{williams_gaussian_2006} B-RBFN~\citep{andrieu_robust_2001}, and standard BNN~\citep{neal_priors_1996}. First row: reference prior; second row: lower lengthscale, same amplitude variance if possible; third row: higher amplitude variance, same lengthscale if possible.}\label{fig:prior_samples}
		\vspace{-0.3cm}
	\end{center}
\end{figure*}

\paragraph{Step 3: Update Poisson process intensity $\L$.}
As stated earlier, the intensity $\L$ has no parametric expression, and cannot be directly evaluated point-wise. However, under the Sigmoid Gaussian Cox Process prior, the intensity $\L(c)=\lambda^\star\phi(h(c))$ only depends on the random variable $h$. Thus, we can build a Monte Carlo point-estimate $\hat{\L}$ by averaging multiple samples from the full conditional posterior of $h$. That is,
\begin{equation}
\hat{\L}(c) \approx \frac{1}{S} \sum  \lambda^\star\phi(h^{(s)}(c)),
\end{equation}
where $h^{(s)} \sim p(h|\y,\x,\ttheta)$.\footnote{Note that $\lambda^\star$ is just a hyper-parameter indicating an upper bound for the intensity.}
We adopt an inference procedure similar to~\citep{adams_gaussian_2008} based on HMC, with the two crucial differences that the centers are unobserved (not fixed) in our case, and that the posterior of $h$ given the $centers$ is \emph{not} independent of the observations.
We add $M$ ``thinned`` auxiliary variables to make computation tractable. We then proceed as follows: i) sample the number of thinned events using death and birth steps, and sample the location of such events using perturbative proposals; ii) sample from the GP posterior $p(h|\y,\x,\c,\ttheta)$.



\section{Results}\label{sec:results}
\vspace{-0.2cm}

In this Section, we empirically demonstrate desirable properties of PoRB-NET. In particular, PoRB-NET allows for (a) easy specification of lengthscale and amplitude variance information (analogous to a GP), and (b) learning of input-dependent lengthscale.  We present synthetic examples and results on three real-case scenarios.
In the Appendix, we report additional empirical results that demonstrate the ability of PoRB-NET to adapt the network architecture based on the data, and to control the uncertainty when extrapolating.  

\vspace{-0.2cm}
\subsection{Synthetic examples}
\vspace{-0.1cm}

\paragraph{PoRB-NET allows for easy specification of stationary lengthscale and signal variance.}
Figure~\ref{fig:prior_samples} shows function samples from different prior models; we decrease the lengthscale from left to right. We plot 50 samples in red, and compute the average variance across the input region averaging out 10,000 function samples. Like a GP, the amplitude variance of PoRB-NET (shown as dotted lines) is constant over the input space and does not depend on the lengthscale. On the other hand, the amplitude variance of B-RBFN \citep{andrieu_robust_2001}, which effectively assumes a homogeneous Poisson process prior on the center parameters, varies over the input space and \textit{does} depend on the lengthscale. The last column shows that for a standard BNN the amplitude variance and lengthscale are concentrated near the origin (within each panel) and that the variance increases as we decrease the lengthscale (from 1st to 2nd row).  



\begin{figure}
	\vspace{-0.3cm}
	\begin{center}
		\includegraphics[width=0.5\textwidth]{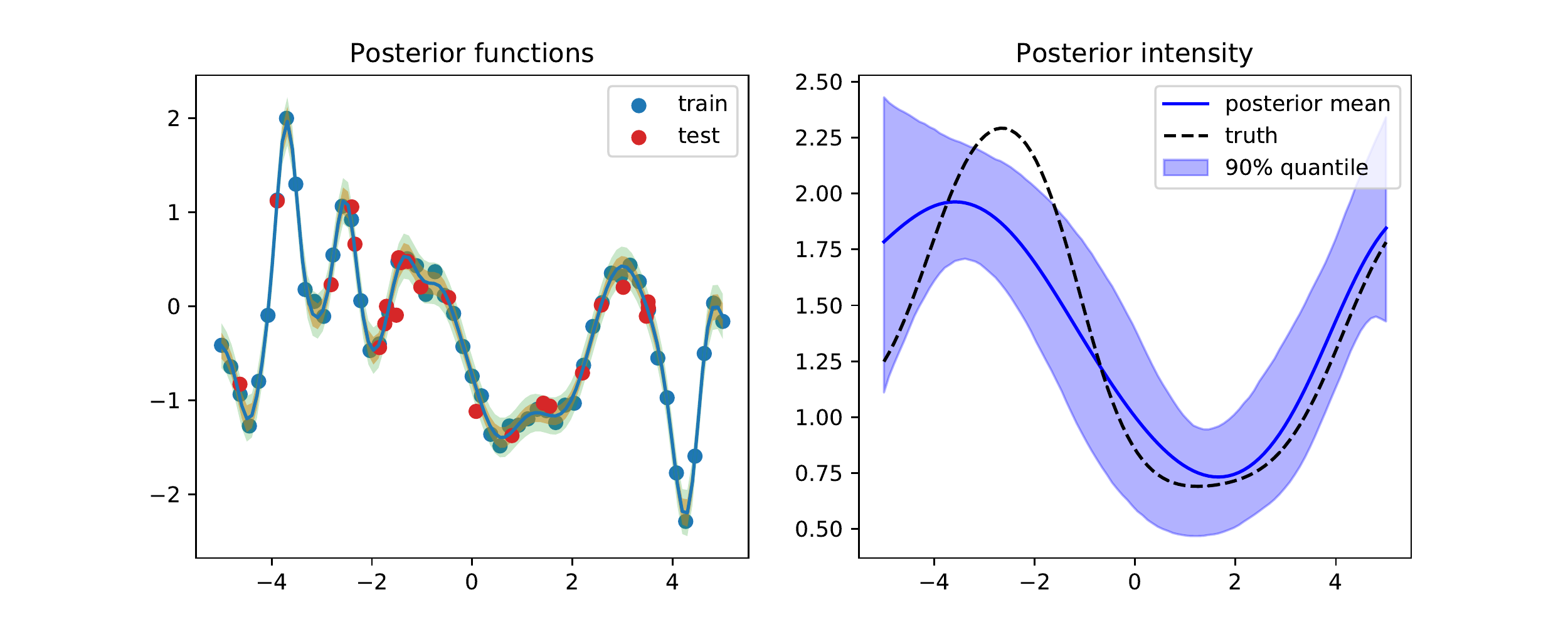}
	\end{center}
	\vspace{-0.3cm}
	\caption{\textbf{PoRB-NET is able to learn input-dependent lengthscale information.} The ground truth synthetic example has been generated by sampling from the PoRB-NET prior.}\label{fig:nonhomoPP}
	\vspace{-0.5cm}
\end{figure}

\begin{figure*}[!b]
	\centering
	\includegraphics[width=1.0\textwidth]{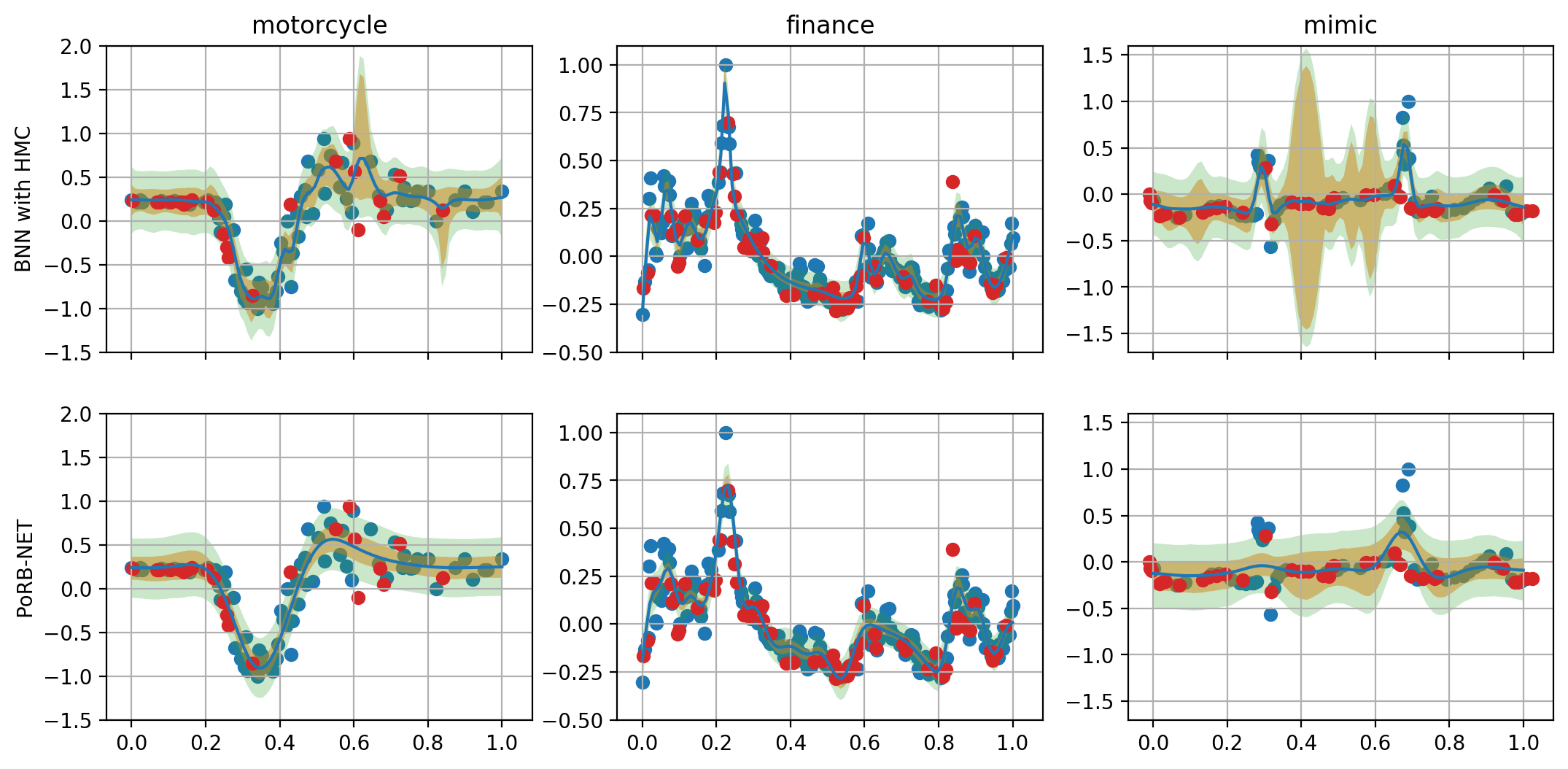}
	\caption{\textbf{PoRB-NET is able to capture non-stationary patterns in real scenarios, adapting the lengthscale locally as needed.} Posterior predictive of PoRB-NET in three real datasets, in comparison to a standard BNN trained with HMC. Priors of both models have been matched in the a priori number of upcrossings; BNN posterior exhibits undesired large fluctuations. Blue points are train set and red points are test set.}\label{fig:realdb_bnn_hmc}
\end{figure*}

\vspace{-0.1cm}
\paragraph{PoRB-NET can capture functions with input-dependent lengthscale.} Figure~\ref{fig:nonhomoPP} illustrates the capacity of PoRB-NET to infer input-dependent lengthscale information. Here the true function is a sample from the PoRB-NET model. The right panel shows that we are able to infer the true intensity function from the noisy $(\x, \y)$ observations. The left panel shows the inferred posterior predictive function.


 
 \vspace{-0.1cm}
\subsection{Real data}
\vspace{-0.1cm}

In this section, we show results for PoRB-NET in three different real-world datasets: motorcycle, mimic, and CBOE volatility index.
All these datasets correspond to non-stationary timeseries (see the Appendix for more information). Figure~\ref{fig:realdb_bnn_hmc} compares the posterior predictive densities of PoRB-NET against BNN trained with HMC.
While PoRB-NET and BNN are matched to have the same number of expected prior upcrossings of y=0 (as a proxy for the inverse lengthscale), the distribution of upcrossings in the posterior is different. This results in the undesired visibly large fluctuations exhibited by the BNN for the motorcycle and mimic datasets, yielding worse test log likelihood and test root mean square error, as shown in Tables~\ref{tab:LLH} and~\ref{tab:RMSE}.
Note that HMC is a gold standard for posterior inference; the fact that the standard BNN lacks desireable properties under HMC demonstrates that its failings come from the model/prior, not the inference.
In the Appendix, we also include further baselines, root mean square errors, held-out log likelihood values, as well as the corresponding inferred intensity functions for PoRB-NET. Interestingly, the learned intensity picks whenever the function exhibit faster variations (see the corresponding figure with the learned intensity function in the Appendix).




\begin{table}[h]
	\centering
	\scriptsize
	\begin{tabular}{llll}
		\toprule
		{} &     \textbf{motorcycle} &      \textbf{finance} &          \textbf{mimic} \\
		\toprule
		\textbf{BNN (HMC)}     &  -0.58 $\pm$ 0.26 &  \textbf{0.55 $\pm$ 0.67} &  -0.52 $\pm$ 0.29 \\
		\textbf{PORBNET} &  \textbf{-0.03 $\pm$ 0.14} &  0.15 $\pm$ 1.02 &   \textbf{0.04 $\pm$ 0.33} \\
		\bottomrule
	\end{tabular}
	\caption{Test Log Likelihood (LLH).}\label{tab:LLH}
\end{table}

\begin{table}[h]
	\centering
	\scriptsize
	\begin{tabular}{llll}
		\toprule
		{} &    \textbf{motorcycle} &      \textbf{finance} &         \textbf{mimic} \\
		\toprule
		\textbf{BNN (HMC)}     &  0.25 $\pm$ 0.02 &  \textbf{0.09 $\pm$ 0.02} &  \textbf{0.14 $\pm$ 0.06} \\
		\textbf{PORBNET} &  \textbf{0.22 $\pm$ 0.02} &   0.1 $\pm$ 0.03 &  0.18 $\pm$ 0.06 \\
		\bottomrule
	\end{tabular}
	\caption{Test Root Mean Square Error (RMSE)}\label{tab:RMSE}
\end{table}

\vspace{-0.2cm}
\section{Discussion}
\label{sec:discussion}

\vspace{-0.2cm}
In this work, we introduced PoRB-NETs, which maintain desirable characteristics of working with neural network models while providing the ability of specifying basic properties such as lengthscale, amplitude variance, and stationarity \emph{in a decoupled manner}. 
While we provide a principled inference scheme, we emphasize that our primary goal was to develop a model that exhibited the appropriate properties; given this, future work can now consider questions of scalable inference.  

While we put priors on the weight space, an important element of our work was reparameterizing the standard BNN formulation to have the center-and-scale form of Eq.~\eqref{eq:RBFN}.  Since the centers lie in the same space as the data, this parameterization makes it much more natural to think about properties such as stationarity, which depend on the data. Additionally, it makes possible to put a Poisson process prior over the hidden units, which facilitates the decoupling of lengthscale and signal variance information.  We focused on radial basis activations because they have a limited region of effect, unlike other popular activations such as Tanh or ReLu.  Exploring how to get desirable properties for those activations seems challenging,
and remains an area for future exploration.
That said, we emphasize that RBFNs are commonly used in many practical applications, as surveyed in~\citet{dash_radial_2016}.


Finally, all of our work was developed in the context of single-layer networks.  From a theoretical perspective, this is not an overly restrictive assumption, as single layer networks are still universal function approximators~\cite{park_universal_1991}; moreover, if prior to our work, we could not get desired properties from single-layer networks, it seemed premature to consider the multi-layer case.  The multi-layer case is an interesting direction for future work, and could perhaps draw on work related to the theoretical properties of deep GPs~\cite{damianou_deep_2012}.

\vspace{-0.1cm}
\section{Conclusion}\label{sec:conclusion}
\vspace{-0.2cm}

This work presents a novel Bayesian prior for neural networks called PoRB-NET that allows for easy encoding of essential basic properties such as lengthscale, signal variance, and stationarity in a decoupled fashion.  
Given the popularity of neural networks and the need for uncertainty quantification in them, understanding prior assumptions---which will govern how we will quantify uncertainty---is essential.  If prior assumptions are not well understood and properly specified, the Bayesian framework makes little sense: the posteriors we find may not be ones that we expect or want.  Our work provides an important step toward specifying Bayesian NN priors with desired basic functional properties. 




%


\bibliography{centered_bnn}

\begin{thebibliography}{}

\bibitem[Adams et~al., 2009]{adams_tractable_2009}
Adams, R.~P., Murray, I., and MacKay, D. J.~C. (2009).
\newblock Tractable nonparametric {Bayesian} inference in {Poisson} processes
  with {Gaussian} process intensities.
\newblock In {\em Proceedings of the 26th {Annual} {International} {Conference}
  on {Machine} {Learning} - {ICML} '09}, pages 1--8, Montreal, Quebec, Canada.
  ACM Press.

\bibitem[Adams and Stegle, 2008]{adams_gaussian_2008}
Adams, R.~P. and Stegle, O. (2008).
\newblock Gaussian process product models for nonparametric nonstationarity.
\newblock In {\em Proceedings of the 25th international conference on {Machine}
  learning - {ICML} '08}, pages 1--8, Helsinki, Finland. ACM Press.

\bibitem[Andrieu et~al., 2001]{andrieu_robust_2001}
Andrieu, C., Freitas, N.~d., and Doucet, A. (2001).
\newblock Robust {Full} {Bayesian} {Learning} for {Radial} {Basis} {Networks}.
\newblock {\em Neural Computation}, 13(10):2359--2407.

\bibitem[Anil et~al., 2018]{anil_sorting_2018}
Anil, C., Lucas, J., and Grosse, R. (2018).
\newblock Sorting out {Lipschitz} function approximation.
\newblock {\em arXiv:1811.05381 [cs, stat]}.
\newblock arXiv: 1811.05381.

\bibitem[Barber and Schottky, 1998]{barber_radial_1998}
Barber, D. and Schottky, B. (1998).
\newblock Radial {Basis} {Functions}: {A} {Bayesian} {Treatment}.
\newblock In Jordan, M.~I., Kearns, M.~J., and Solla, S.~A., editors, {\em
  Advances in {Neural} {Information} {Processing} {Systems} 10}, pages
  402--408. MIT Press.

\bibitem[Cedilnik et~al., 2004]{cedilnik_distribution_2004}
Cedilnik, A., Kosmelj, K., and Blejec, A. (2004).
\newblock The distribution of the ratio of jointly normal variables.
\newblock {\em Metodoloski zvezki}, 1(1):99.

\bibitem[Damianou and Lawrence, 2012]{damianou_deep_2012}
Damianou, A.~C. and Lawrence, N.~D. (2012).
\newblock Deep {Gaussian} {Processes}.
\newblock {\em arXiv:1211.0358 [cs, math, stat]}.
\newblock arXiv: 1211.0358.

\bibitem[Dash et~al., 2016]{dash_radial_2016}
Dash, C. S.~K., Behera, A.~K., Dehuri, S., and Cho, S.-B. (2016).
\newblock Radial basis function neural networks: a topical state-of-the-art
  survey.
\newblock {\em Open Computer Science}, 6(1).

\bibitem[Doob, 1949]{doob_application_1949}
Doob, J. (1949).
\newblock Application of the theory of martingales.
\newblock {\em Acted du Colloque International Le Calcul des Probabilités et
  ses applications}.

\bibitem[Flam-Shepherd et~al., 2017]{flam-shepherd_mapping_2017}
Flam-Shepherd, D., Requeima, J., and Duvenaud, D. (2017).
\newblock Mapping {Gaussian} {Process} {Priors} to {Bayesian} {Neural}
  {Networks}.
\newblock {\em Bayesian Deep Learning Workshop NIPS 2017}.

\bibitem[Freedman, 1963]{freedman_asymptotic_1963}
Freedman, D.~A. (1963).
\newblock On the {Asymptotic} {Behavior} of {Bayes}' {Estimates} in the
  {Discrete} {Case}.
\newblock {\em The Annals of Mathematical Statistics}, 34(4):1386--1403.

\bibitem[Gal and Ghahramani, 2015]{gal_dropout_2015}
Gal, Y. and Ghahramani, Z. (2015).
\newblock Dropout as a {Bayesian} {Approximation}: {Representing} {Model}
  {Uncertainty} in {Deep} {Learning}.
\newblock {\em arXiv:1506.02142 [cs, stat]}.
\newblock arXiv: 1506.02142.

\bibitem[Garnelo et~al., 2018]{garnelo_neural_2018}
Garnelo, M., Schwarz, J., Rosenbaum, D., Viola, F., Rezende, D.~J., Eslami, S.
  M.~A., and Teh, Y.~W. (2018).
\newblock Neural {Processes}.
\newblock {\em arXiv:1807.01622 [cs, stat]}.
\newblock arXiv: 1807.01622.

\bibitem[Györfi et~al., 2002]{gyorfi_distribution-free_2002}
Györfi, L., Kohler, M., Krzyżak, A., and Walk, H. (2002).
\newblock {\em A {Distribution}-{Free} {Theory} of {Nonparametric}
  {Regression}}.
\newblock Springer {Series} in {Statistics}. Springer New York, New York, NY.

\bibitem[Heinonen et~al., 2015]{heinonen_non-stationary_2015}
Heinonen, M., Mannerström, H., Rousu, J., Kaski, S., and Lähdesmäki, H.
  (2015).
\newblock Non-{Stationary} {Gaussian} {Process} {Regression} with {Hamiltonian}
  {Monte} {Carlo}.
\newblock {\em arXiv:1508.04319 [stat]}.
\newblock arXiv: 1508.04319.

\bibitem[Hernández-Lobato et~al., 2014]{hernandez-lobato_predictive_2014}
Hernández-Lobato, J.~M., Hoffman, M.~W., and Ghahramani, Z. (2014).
\newblock Predictive {Entropy} {Search} for {Efficient} {Global} {Optimization}
  of {Black}-box {Functions}.
\newblock {\em arXiv:1406.2541 [cs, stat]}.
\newblock arXiv: 1406.2541.

\bibitem[Hinton and Neal, 1995]{hinton_bayesian_1995}
Hinton, G.~E. and Neal, R.~M. (1995).
\newblock Bayesian learning for neural networks.

\bibitem[Holmes and Mallick, 1998]{holmes_bayesian_1998}
Holmes, C.~C. and Mallick, B.~K. (1998).
\newblock Bayesian radial basis functions of variable dimension.
\newblock {\em Neural computation}, 10(5):1217--1233.

\bibitem[Izmailov et~al., 2019]{izmailov_subspace_2019}
Izmailov, P., Maddox, W.~J., Kirichenko, P., Garipov, T., Vetrov, D., and
  Wilson, A.~G. (2019).
\newblock Subspace {Inference} for {Bayesian} {Deep} {Learning}.
\newblock {\em arXiv:1907.07504 [cs, stat]}.
\newblock arXiv: 1907.07504.

\bibitem[Kingman, 1992]{kingman_poisson_1992}
Kingman, J. F.~C. (1992).
\newblock {\em Poisson {Processes}}.
\newblock Clarendon Press.

\bibitem[LeCun et~al., 2015]{lecun_deep_2015}
LeCun, Y., Bengio, Y., and Hinton, G. (2015).
\newblock Deep learning.
\newblock {\em Nature}, 521(7553):436--444.

\bibitem[Lee, 2004]{lee_bayesian_2004}
Lee, H. (2004).
\newblock {\em Bayesian {Nonparametrics} via {Neural} {Networks}}.
\newblock {ASA}-{SIAM} {Series} on {Statistics} and {Applied} {Mathematics}.
  Society for Industrial and Applied Mathematics.

\bibitem[Lee, 2000]{lee_consistency_2000}
Lee, H.~K. (2000).
\newblock Consistency of posterior distributions for neural networks.
\newblock {\em Neural Networks: The Official Journal of the International
  Neural Network Society}, 13(6):629--642.

\bibitem[Lee, 2003]{lee_noninformative_2003}
Lee, H.~K. (2003).
\newblock A noninformative prior for neural networks.
\newblock {\em Machine Learning}, 50(1-2):197--212.

\bibitem[Lippmann, 1989]{lippmann_pattern_1989}
Lippmann, R.~P. (1989).
\newblock Pattern classification using neural networks.
\newblock {\em IEEE communications magazine}, 27(11):47--50.

\bibitem[MacKay, 1992]{mackay_bayesian_1992}
MacKay, D.~J. (1992).
\newblock {\em Bayesian methods for adaptive models}.
\newblock {PhD} {Thesis}, California Institute of Technology.

\bibitem[Miller, 2018]{miller_detailed_2018}
Miller, J.~W. (2018).
\newblock A detailed treatment of {Doob}'s theorem.
\newblock {\em arXiv:1801.03122 [math, stat]}.
\newblock arXiv: 1801.03122.

\bibitem[Møller et~al., 1998]{moller_log_1998}
Møller, J., Syversveen, A.~R., and Waagepetersen, R.~P. (1998).
\newblock Log gaussian cox processes.
\newblock {\em Scandinavian journal of statistics}, 25(3):451--482.

\bibitem[Müller and Insua, 1998]{muller_issues_1998}
Müller, P. and Insua, D.~R. (1998).
\newblock Issues in {Bayesian} analysis of neural network models.
\newblock {\em Neural Computation}, 10(3):749--770.

\bibitem[Neal, 1996]{neal_priors_1996}
Neal, R.~M. (1996).
\newblock Priors for infinite networks.
\newblock In {\em Bayesian {Learning} for {Neural} {Networks}}, pages 29--53.
  Springer.

\bibitem[Park and Sandberg, 1991]{park_universal_1991}
Park, J. and Sandberg, I.~W. (1991).
\newblock Universal {Approximation} {Using} {Radial}-{Basis}-{Function}
  {Networks}.
\newblock {\em Neural Computation}, 3(2):246--257.

\bibitem[Pearce et~al., 2019]{pearce_expressive_2019}
Pearce, T., Tsuchida, R., Zaki, M., Brintrup, A., and Neely, A. (2019).
\newblock Expressive {Priors} in {Bayesian} {Neural} {Networks}: {Kernel}
  {Combinations} and {Periodic} {Functions}.
\newblock page~11.

\bibitem[Powell, 1987]{powell_algorithms_1987}
Powell, M. J.~D. (1987).
\newblock Algorithms for {Approximation}.
\newblock pages 143--167. Clarendon Press, New York, NY, USA.

\bibitem[Robinson, 2001]{robinson_priors_2001}
Robinson, M. (2001).
\newblock {\em Priors for {Bayesian} {Neural} {Networks}}.
\newblock PhD thesis, University of British Columbia.

\bibitem[Sun et~al., 2019]{sun_functional_2019}
Sun, S., Zhang, G., Shi, J., and Grosse, R. (2019).
\newblock Functional {Variational} {Bayesian} {Neural} {Networks}.
\newblock page~23.

\bibitem[Thompson, 1933]{thompson_likelihood_1933}
Thompson, W.~R. (1933).
\newblock On the likelihood that one unknown probability exceeds another in
  view of the evidence of two samples.
\newblock {\em Biometrika}, 25(3/4):285--294.

\bibitem[van~der Vaart and Wellner, 1996]{vandervaart_convergence_1996}
van~der Vaart, A.~W. and Wellner, J.~A. (1996).
\newblock {\em Weak convergence and empirical processes}.
\newblock Springer. Springer New York, New York, NY.

\bibitem[Wainwright, 2019]{wainwright_high-dimensional_2019}
Wainwright, M.~J. (2019).
\newblock {\em High-dimensional statistics: {A} non-asymptotic viewpoint},
  volume~48.
\newblock Cambridge University Press.

\bibitem[Watanabe, 2007]{watanabe_almost_2007}
Watanabe, S. (2007).
\newblock Almost {All} {Learning} {Machines} are {Singular}.
\newblock In {\em 2007 {IEEE} {Symposium} on {Foundations} of {Computational}
  {Intelligence}}, pages 383--388.

\bibitem[Williams, 1997]{williams_computing_1997}
Williams, C.~K. (1997).
\newblock Computing with infinite networks.
\newblock In {\em Advances in neural information processing systems}, pages
  295--301.

\bibitem[Williams and Rasmussen, 2006]{williams_gaussian_2006}
Williams, C.~K. and Rasmussen, C.~E. (2006).
\newblock {\em Gaussian processes for machine learning}, volume~2.
\newblock MIT press Cambridge, MA.

\bibitem[Wong and Shen, 1995]{wong_sieves_1995}
Wong, W.~H. and Shen, X. (1995).
\newblock Probability inequalities for likelihood ratios and convergence rates
  of sieve mles.
\newblock (2).

\bibitem[Yang, 2005]{yang_bayesian_2005}
Yang, Z.~R. (2005).
\newblock Bayesian {Radial} {Basis} {Function} {Neural} {Network}.
\newblock In Gallagher, M., Hogan, J.~P., and Maire, F., editors, {\em
  Intelligent {Data} {Engineering} and {Automated} {Learning} - {IDEAL} 2005},
  Lecture {Notes} in {Computer} {Science}, pages 211--219. Springer Berlin
  Heidelberg.

\end{thebibliography}

\newpage
\onecolumn

%


\setcounter{section}{0}
\section{Appendix: Covariance}

In this Section, we derive the covariance function of PoRB-NET. We consider one dimensional inputs for simplicity. First, we show that our model has zero-mean prior. Note that $b$, $\{(w_k,c_k)\}_{k=1}^K$, and $K$ are all random variables, and the scales $s^2_k$ are fixed as a function of the intensity, e.g., $s^2_k =s^2_0  \lambda(c_k)^2$. Let $\phi$ denote an arbitrary activation function, e.g., $\phi(x) = \exp(-x)$.
\begin{align}
\E[f(x)] &= \E\left[b+\sum_{k=1}^K w_k \phi(s_k(x-c_k)) \right] \\
&= \E[b] + \E\left[\sum_{k=1}^K w_k \phi(s_k(x-c_k)) \right] \\
&= \E\left[ \E\left[\sum_{k=1}^K w_h \phi(s_k(x-c_k)) \mid K=K_0 \right] \right] \\
&= \sum_{K_0=0}^\infty \text{Pr}[K=K_0] \E\left[\sum_{k=1}^K w_k \phi(s_k(x-c_k)) \mid K=K_0 \right]  \\
&= \sum_{K_0=0}^\infty \text{Pr}[K=K_0] \sum_{k=1}^{K_0} \E\left[ w_k \phi(s_k(x-c_k)) \mid K=K_0 \right]  \\
&= \sum_{K_0=0}^\infty \text{Pr}[K=K_0] \sum_{k=1}^{K_0} \E\left[ w_k \phi(s_k(x-c_k)) \right]  \label{eq:drop_K0}\\
&= \sum_{K_0=0}^\infty \text{Pr}[K=K_0] K_0 \E\left[ w_k \phi(s_k(x-c_k)) \right]  \\
&=  \E\left[ w_k \phi(s_k(x-c_k)) \right] \sum_{K_0=0}^\infty  \text{Pr}[K=K_0] K_0 \\
&=  \underbrace{\E\left[ w_k \right]}_{0} \E\left[ \phi(s_k(x-c_k)) \right] \E[K_0] \\
&= 0
\end{align}
In Equation~\eqref{eq:drop_K0}, we drop the condition $K=K_0$ since, conditional on the network width $K$ being fixed, the weights $\{w_k\}$ are independently normally distributed and the centers $\{c_k\}$ are independently distributed according to the normalized intensity $\lambda(c) / \Lambda$ where $\Lambda = \int_{\mathcal{C}} \lambda(c) dc$, so they do not depend on the actual value of the network width. 

Next, we consider the covariance:
\begin{align}
\text{Cov}\left[f(x_1), f(x_2) \right] &= \E\left[f(x_1) f(x_2) \right] \nonumber \\
&= \E\left[ \left(b+\sum_{k=1}^K w_k \phi(s_k(x_1-c_k)) \right) \left(b+\sum_{k=1}^K w_k \phi(s_k(x_2-c_k)) \right)  \right] \\
&= \E\left[ \E\left[ \left(b+\sum_{k=1}^K w_k \phi(s_k(x_1-c_k)) \right) \left(b+\sum_{k=1}^K w_k \phi(s_k(x_2-c_k)z) \right) \mid K=K_0 \right] \right] \nonumber \\
&= \E\left[ \E\left[  b^2 + \sum_{k_1=1}^K \sum_{k_2=1}^K w_{k_1} w_{k_2} \phi(s_{k_1}(x_1-c_{k_1})) \phi(s_{k_2}(x_2-c_{k_2}))  \mid K=K_0 \right] \right] \nonumber \\
&= \sigma^2_b + \E\left[\E\left[ \sum_{k=1}^K w_k^2 \phi(s_k(x_1-c_k)) \phi(s_k(x_2-c_k)) \mid K=K_0\right] \right] \nonumber \\
&\quad +  \E\left[\E\left[ 2 \sum_{k_1=1}^K \sum_{k_2=k_1+1}^K w_{k_1} w_{k_2} \phi(s_{k_1}(x_1-c_{k_1})) \phi(s_{k_2}(x_2-c_{k_2})) \mid K=K_0\right] \right] \nonumber \\
&= \sigma^2_b + \E\left[\sum_{k=1}^{K_0} \underbrace{\E\left[ w_k^2\right]}_{\tilde{\sigma}^2_w} \E\left[\phi(s_k(x_1-c_k)) \phi(s_k(x_2-c_k)) \mid K=K_0\right] \right] \nonumber \\
&\quad +  \E\left[ 2 \sum_{k_1=1}^K \sum_{k_2=k_1+1}^K \underbrace{\E\left[w_{k_1}\right]}_{0} \underbrace{\E\left[w_{k_2}\right]}_{0} \E\left[\phi(s_{k_1}(x_1-c_{k_1})) \phi(s_{k_2}(x_2-c_{k_2})) \mid K=K_0\right] \right] \nonumber \\
&= \sigma^2_b + \E\left[\sum_{k=1}^{K_0} \tilde{\sigma}^2_w \E\left[\phi(s_k(x_1-c_k)) \phi(s_k(x_2-c_k)) \mid K=K_0\right] \right] \\
&= \sigma^2_b + \E\left[\sum_{k=1}^{K_0} \tilde{\sigma}^2_w \E\left[\phi(s_k(x_1-c_k)) \phi(s_k(x_2-c_k)) \right] \right] \label{eq:drop_K0_2} \\
&= \sigma^2_b + \E\left[K_0 \tilde{\sigma}^2_w \E\left[\phi(s(x_1-c)) \phi(s(x_2-c)) \right] \right] \label{eq:drop_k} \\
&= \sigma^2_b + \tilde{\sigma}^2_w \E\left[K_0 \right] \E\left[\phi(s(x_1-c)) \phi(s(x_2-c))\right], \label{eq:K_0}  
\end{align}
where $\tilde{\sigma}^2_w = \sqrt{s^2_0 / \pi} \sigma^2_w$ is the prior variance for the weights.
In Equation (\ref{eq:drop_K0_2}), as in Equation~\eqref{eq:drop_K0}, we drop the condition $K=K_0$ but remember that the expectation is with respect to the weight and center parameters conditional on a fixed network width.
Also note that going from Equation~\eqref{eq:drop_K0} to Equation~\eqref{eq:drop_k} relies on the fact that the priors on the scales $\{s_k\}$ and centers $\{c_k\}$ are i.i.d priors; we drop the subindex $k$ from now on for notational simplicity.

To actually evaluate the covariance, we need to evaluate the term $\E\left[\phi(s(x_1-c)) \phi(s(x_2-c)) \right]$. We next consider two cases. Case 1 is a homogeneous Poisson process prior over $c$ and Case 2 is an inhomogeneous Poisson process prior over $c$. Note that in both cases, the Poisson process prior over $c$ is unconditional on the network width. Conditioned on the network width, as in the expectation that we are trying to evaluate, Case 1 is a uniform distribution over $\mathcal{C}$ and Case 2 has PDF $\lambda(c)/\Lambda$.

\subsection*{Case 1: Homogeneous Poisson Process}

First we consider the case where the intensity is fixed, i.e., $\lambda(c) = \lambda$. Then we have:
\begin{align}
\E&\left[\phi(s(x_1-c)) \phi(s(x_2-c))\right]  \\
&= \int_{\mathcal{C}} \phi(s(x_1-c)) \phi(s(x_2-c)) \frac{\lambda}{\Lambda}~dc  \\
&=  \int_{\mathcal{C}} \exp\left\{-\frac{1}{2}(s(x_1-c))^2 \right\} \exp\left\{-\frac{1}{2}(s(x_1-c))^2 \right\}  \frac{\lambda}{\Lambda} ~dc  \\
&= \int_{\mathcal{C}} \exp\left\{-\frac{1}{2}s^2 [(x_1-c)^2+(x_2-c)^2] \right\}  \frac{\lambda}{\Lambda} ~dc   \\
&=\int_{\mathcal{C}} \exp\left\{-s^2 \left[\left(\frac{x_1-x_2}{2} \right)^2 +  \left(\frac{x_1+x_2}{2} - c\right)^2  \right] \right\} \frac{\lambda}{\Lambda}  ~dc  \\
&=\int_{\mathcal{C}} \exp\left\{-s^2 \left(\frac{x_1-x_2}{2} \right)^2\right\} \exp\left\{-s^2 \left[ \left(\frac{x_1+x_2}{2} - c\right)^2 \right] \right\}  \frac{\lambda}{\Lambda} ~dc   \\
&=\int_{\mathcal{C}} \exp\left\{-s^2_0 \lambda^2 \left(\frac{x_1-x_2}{2} \right)^2\right\} \exp\left\{-s^2_0 \lambda^2 \left[ \left(\frac{x_1+x_2}{2} - c\right)^2 \right] \right\}  \frac{\lambda}{\Lambda} ~dc  \label{eq:h_lambda} \\
&=  \underbrace{\exp\left\{-s^2_0 \lambda^2 \left(\frac{x_1-x_2}{2} \right)^2\right\}}_{\text{SE kernel}}  \underbrace{\int_{\mathcal{C}} \exp\left\{-s^2_0 \lambda^2 \left[ \left(\frac{x_1+x_2}{2} - c\right)^2 \right] \right\}  \frac{\lambda}{\Lambda} ~dc}_{ \text{uniform mixture of Gaussians} } \label{eq:kernel_mixture}
\end{align}
In Equation~\eqref{eq:h_lambda}, we plug in $s^2 = s^2_0 \lambda(c)^2 =  s^2_0 \lambda^2$. In Equation~\eqref{eq:kernel_mixture}, we point out that we can write this term as the product of an square exponential kernel and a mixture of Gaussians. Considering only the uniform mixture of Gaussian terms, we have:
\begin{align}
\int_{\mathcal{C}}  \exp\left\{-s^2_0 \lambda^2 \left[ \left(\frac{x_1+x_2}{2} - c\right)^2 \right] \right\} \frac{\lambda}{\Lambda} ~dc
&= \frac{\lambda}{\Lambda} \int_{\mathcal{C}} \exp\left\{-s^2_0 \lambda^2 \left[ \left(\frac{x_1+x_2}{2} - c\right)^2 \right] \right\} ~dc \\
&=  \frac{\lambda}{\Lambda} \int_{C_0}^{C_1} \exp\left\{-\frac{1}{2\psi^2} \left[ \left(x_m - c\right)^2 \right] \right\} ~dc \label{eq:psi} \\
&=  \frac{\lambda}{\Lambda} \psi  \sqrt{2\pi} \int_{(C_0-x_m)/\psi}^{(C_1 - x_m)\psi} \frac{1}{\sqrt{2\pi}} \exp\left\{-\frac{1}{2} u^2  \right\} ~du \label{eq:var_change}\\
&= \frac{\lambda}{\Lambda} \frac{1}{\sqrt{2 s_0^2 \lambda^2}} \sqrt{2\pi} \left[\Phi((C_1-x_m) / \psi) - \Phi((C_1-x_m) / \psi)\right] \\
&= \frac{1}{\Lambda} \sqrt{\frac{\pi}{s^2_0}} \left[\Phi((C_1-x_m) \sqrt{2}s_0\lambda) - \Phi((C_0-x_m) \sqrt{2}s_0\lambda)\right] \label{eq:mixture}
\end{align}
where $\Phi$ is the error function for a standard Gaussian, and $\mathcal{C}=[C_0,C_1]$ is the region where the Poisson Process intensity is defined.
In Equation (\ref{eq:psi}) we define $\psi^2:=1/(2 s_0^2 \lambda^2)$, and $x_m := (x_1+x_2)/2$ as the midpoint. In Equation (\ref{eq:var_change}) we use the change of variables $u = (c - x_m)/\sigma$. Plugging Equation~(\ref{eq:mixture}) in Equation~(\ref{eq:kernel_mixture}), and Equation (\ref{eq:kernel_mixture}) into Equation (\ref{eq:K_0}), we have:
\begin{align}
\text{Cov}\left[f(x_1), f(x_2) \right] &= 
\sigma^2_b + \sigma^2_w 
\exp\left\{-s^2_0 \lambda^2 \left(\frac{x_1-x_2}{2} \right)^2\right\} \left[\Phi((C_1-x_m) \sqrt{2}s_0\lambda) - \Phi((C_0-x_m) \sqrt{2}s_0\lambda)\right].
\end{align}
This gives a closed form representation for the covariance (to the extent that the error function is closed form). 
If we further assume that $C_0$ and $C_1$ are large in absolute value relative to the midpoint $x_m$, in other words, that the Poisson Process is defined over a larger region than the data, then the difference in error functions is approximately one (i.e., the integral over the tails of the Gaussian goes to zero) and the covariance becomes:
\begin{align}
\text{Cov}\left[f(x_1), f(x_2) \right] &\approx 
\sigma^2_b + \sigma^2_w 
\exp\left\{-s^2_0 \lambda^2 \left(\frac{x_1-x_2}{2} \right)^2\right\}.
\end{align}
Finally, notice that the variance depends only on the weight and bias variance parameters:
\begin{align}
\mathbb{V}\left[f(x) \right] &\approx 
\sigma^2_b + \sigma^2_w 
\end{align}

\subsection*{Case 2: Inhomogeneous Poisson Process}

Now consider an inhomogenous Poisson process prior on the center parameters with an arbitrary intensity function $\lambda(c)$. We use the heuristic $s^2_k = s^2_0 \lambda(c_k)$. Figure \ref{fig:heuristic} shows different function samples for a single fixed intensity sampled from the prior. On the left, the $s^2_k$ is constant for each unit while on the right $s^2_k = s^2_0 \lambda(c_k)$. The top row shows the true intensity, the middle row shows the amplitude/function variance, and the bottom row shows a histogram of the number of function upcrossings of zero, which can be thought of as a measurement of lengthscale. We see that setting the scales based on the intensity results in approximately constant function variance and increased number of upcrossings. 

\begin{figure}[H]
	\centering
	\begin{subfigure}[t]{0.5\textwidth}
		\centering
		\includegraphics[height=3.4in]{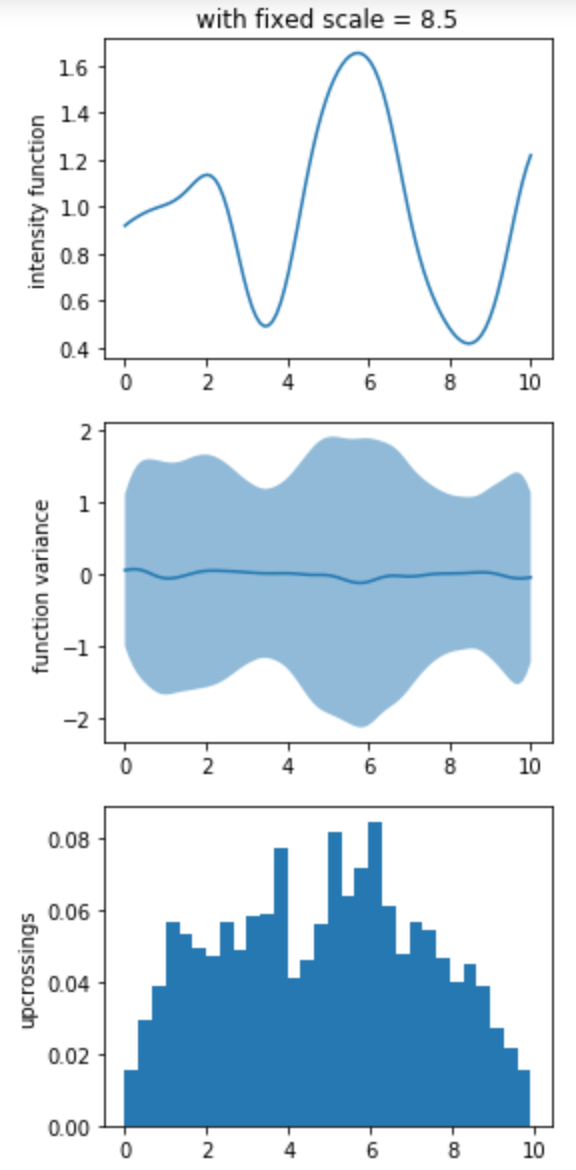}
		\caption{Constant $s_k^2$}
	\end{subfigure}%
	~
	\begin{subfigure}[t]{0.5\textwidth}
		\centering
		\includegraphics[height=3.4in]{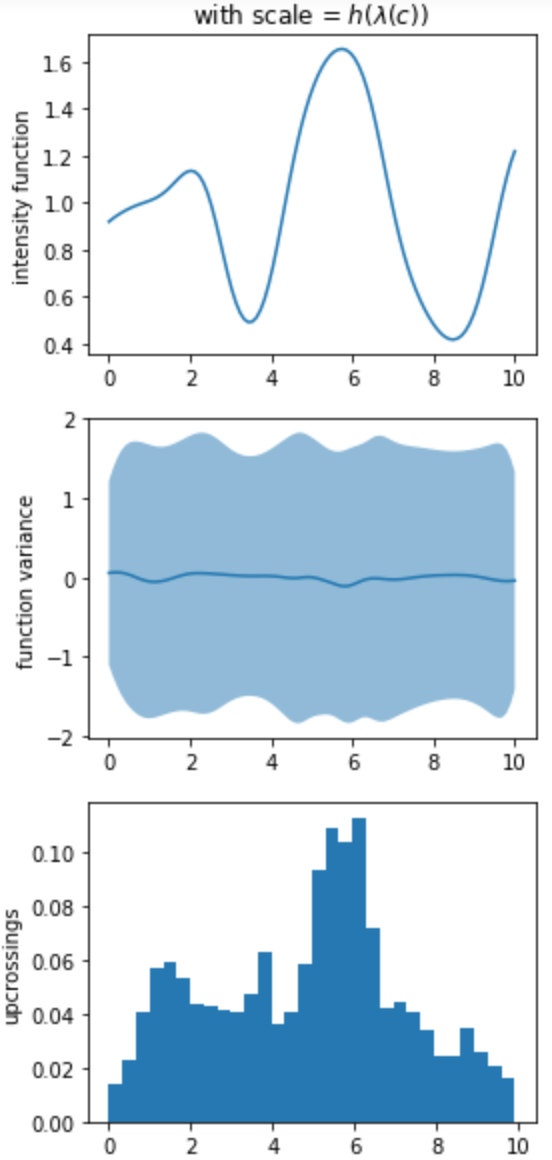}
		\caption{$s_k^2= s^2_0 \lambda(c_k)^2$, where for $s^2_0 = 2\pi$.}
	\end{subfigure}
	\caption{Setting $s_k^2= s^2_0 \lambda(c_k)^2$ results in approximately stationary amplitude variance. \label{fig:heuristic}}
\end{figure}

\newpage

\section{Appendix: Proof for consistency}

We are interested in the posterior behavior of our model as the number of observations $n\to\infty$. Specifically, we want to show that our estimated regression function $\hat{g}_n(x) = \mathbb{E}[Y\mid X=x]$ is asymptotically consistent for the true regression function $g_0(x)$, i.e.:
$$
\int (\hat{g}_n(x) - g_0(x))^2 dx \xrightarrow{p} 0
$$
To do this, we first show that the posterior probability assigned to all joint distribution functions $f(X,Y)$ in any Hellinger neighborhood of the true joint distribution function $f_0(X,Y)$ approaches one as $n\to\infty$. That is, if $A_\epsilon = \{f\mid D_H(f,f_0) \le \epsilon \}$ defines a Hellinger neighborhood of the true distribution function, then $\forall \epsilon>0$:
$$
p(A_\epsilon \mid (X_1,Y_1),\dotsc,(X_n,Y_n))  \xrightarrow{p} 1
$$  
We assume that the marginal distribution of $X$ is uniform on $[0,1]$ (i.e., $f(X)=1$), so the joint distribution $f(X,Y)$ and the conditional distribution $f(Y\mid X)$ are the same, since $f(X,Y) = f(Y\mid X) f(X) = f(Y\mid X)$. The estimated regression function is defined as $\hat{g}_n(x) = \mathbb{E}_{\hat{f}_n}[Y\mid X=x]$, where $\hat{f}_n$ is given by the posterior predictive density:
$$
\hat{f}_n(X,Y) = \int f(X,Y) ~dP(f \mid (X_1,Y_1)),\dotsc,(X_n,Y_n)).
$$

After introducing a few definitions and notation, Section \ref{sec:consistency_general} discusses the necessary conditions on the prior required for any radial basis function network to achieve consistency, with many results taken or adapted from \citep{lee_consistency_2000}. Section \ref{sec:consistency_porbnet} checks that these necessary conditions are met by PoRB-NET with a homogeneous Poisson process prior on the number of hidden units. We first show asymptotic consistency when the number of hidden units is allowed to grow with the data. This gives a sequence of models known as a sieve. We then extend this to the case when the number of hidden units is inferred. 

\subsection*{Definitions and notation} \label{sec:consistency_notation}
We begin by specifying our notation and definitions:

\begin{itemize}
	
	\item $D$ is the input dimension.
	
	\item $K$ is the network width.
	
	\item $I$, $I^{(w)}$ and $I^{(c)}$ are the number of total parameters, weight parameters, and center parameters, respectively. $I = I^{(w)} + I^{(c)} + 1$. 
	
	\item $\mathcal{I}$, $\mathcal{I}^{(w)}$, $\mathcal{I}^{(c)}$, and $\mathcal{I}^{(\lambda^2)}$ are the index set of total parameters, weight parameters, center parameters, and intensity respectively (e.g., $\mathcal{I} = 1,2,\dotsc,I$). $\mathcal{I}^{(w)} \subset \mathcal{I}$, $\mathcal{I}^{(c)} \subset \mathcal{I}$, $\mathcal{I}^{(\lambda^2)} \subset \mathcal{I}$, $I = |\mathcal{I}|$, $I^{(w)} = |\mathcal{I}^{(w)}|$, $I^{(c)} = |\mathcal{I}^{(c)}|$, and $1=|\mathcal{I}^{(\lambda^2)}|$.
	
	\item The subscript $n$ always denotes the sample size dependence (applies to $K_n$, $I_n$, $\mathcal{I}_n$, $I_n^{(w)}$, $\mathcal{I}_n^{(w)}$, $I_n^{(c)}$, $\mathcal{I}_n^{(c)}$, $C_n$).
	
	\item Let $\theta_i$ denote any parameter, $c_i$ denote a center parameter, and $w_i$ denote a weight parameter.
	
	\item $C_n$ is a bound on the absolute value of the parameters. For the sieves approach in we assume $C_n\le \exp(n^{b-a})$, where $0<a<b<1$.
	
	\item Assume that the Poisson process intensity function $\lambda(c)$ is only defined on a bounded region $\mathcal{C}$.
	
	\item Let $f(x,y)$ denote a joint density of covariates $X$ and label $Y$ and let $g(x) = \mathbb{E}[Y\mid X=x]$ denote a regression function.
	
	\item Let $f_0(x,y)$ and $g_0(x)$ denote the true joint density and regression function, respectively.
	
	\item We assume $x\in\mathcal{X}=[0,1]^D$ and that the marginal density of $x$ is uniform, i.e. $f(X)=1$.
	
	\item Let $D_H(f_0, f)$ denote the Hellinger distance and let $A_\epsilon = \{f : D_H(f_0,f) \le \epsilon\}$.
	\item Let $D_K(f_0, f)$ denote the KL divergence and let $K_\gamma$ denote a KL neighborhood of the true joint density:
	$
	K_\gamma = \{f \mid D_K(f_0,f) \le \gamma \}
	$$ = \{f : D_K(f_0,f) \le \gamma\}$.
	
	\item Let $(x_1,y_1),\dotsc,(x_n,y_n)$ denote the $n$ observations and $\pi_n$ denote a prior probability distribution over the parameters of a single hidden layer PoRB-NET conditional on there being $K_n$ nodes, where $K_n$ increases with $n$. Let $I_n$ denote the number of parameters for an RBFN network with $K_n$ nodes.  
	
	\item Let $\mathcal{F}$ denote the space of all single-layer radial basis function networks $\text{RBFN}(x; \bm{\theta})\mapsto y$, let $\mathcal{F}_n \subset \mathcal{F}$ be its restriction to networks with parameters less than $C_n>0$ in absolute value, where $C_n$ also increases with $n$; let $\mathcal{H}_n \subset \mathcal{F}$ be its restriction to networks with $K_n$ nodes; and let $\mathcal{G}_n = \mathcal{F}_n \cap \mathcal{G}_n$ be the intersection of both restrictions. 
	
\end{itemize}

\subsection{Consistency for RBFNs with arbitrary priors} \label{sec:consistency_general}

\subsubsection{Supporting results}

The following theorems are used in proof of Lemma \ref{lemma:adapt_lemma_2_lee}, which is adapted from \citep{lee_consistency_2000}.
Theorem~\ref{thm:3} upper bounds the bracketing number $N_{[]}(\ )$ by the covering number $N(\ )$. Define the Hellinger bracketing entropy by $H_{[]}(\ ) := \log N_{[]}(\ )$.

\begin{theorem} \label{thm:3}
	\citep{vandervaart_convergence_1996}
	Let $s, t \in \mathcal{F}_n$, i.e., $s$ and $t$ are realizations of the parameter vector. Let $f_t(x,y) \in \mathcal{F}^*$ be a function of $x$ and $y$ with parameter vector equal to $t$. Suppose that:
	\begin{equation}
	\left\lvert f_t(x,y) - f_s(x,y) \right\rvert \le d^*(s,t) F(x,y)
	\end{equation}
	for some metric $d^*$, for some fixed function $F$, and for every $s$, $t$, and every $(x,y)$. Then for any norm $\left\| \cdot \right\|$,
	\begin{equation}
	N_{[]}(2\epsilon \left\| F \right\|, \mathcal{F}^*, \left\| \cdot \right\|) \le N(\epsilon,  \mathcal{F}_n, d^*).
	\end{equation}
	
\end{theorem}

\begin{theorem} \label{thm:4}
	\citep{wong_sieves_1995}
	Define the ratio of joint likelihoods between the inferred density and the true density as
	\begin{equation}
	R_n(f) = \prod_{i=1}^n \frac{f(x_i, y_i)}{f_0(x_i, y_i)}.
	\end{equation}
	For any $\epsilon>0$ there exists constants $a_1$, $a_2$, $a_3$, $a_4$ such that if 
	\begin{equation}
	\int_{\epsilon^2/2^8}^{\sqrt\epsilon} \sqrt{H_{[]}(u/ a_3)}~du \le 2 a_4 \sqrt{n} \epsilon^2,
	\end{equation}
	then
	\begin{equation}
	P^{*} \left( \sup_{f\in A_\epsilon^c \cap \mathcal{F}_n} R_n(f) \ge \exp(-a_1 n \epsilon^2) \right) \le 4\exp(-a_2 n \epsilon^2).
	\end{equation}
	
\end{theorem}

\begin{lemma} (Adaptation of Lemma 1 in \cite{lee_consistency_2000})\footnote{This lemma differs from \citep{lee_consistency_2000} because they assume $H_{[]}(u) \le \log[(C_n^2 I_n/u )^{I_n}]$ and $I_n = (D+2)K_n + 1$.} \label{lemma:adapt_lemma_1_lee}
	Suppose that $H_{[]}(u) \le \log[(a' n^a C_n^{a''} I_n/u )^{I_n}]$, where $I_n = (D+1)K_n + 1$, $K_n \le n^a$, $a',a''>0$, and $C_n \le \exp(n^{b-a})$ for $0<a<b<1$. Then for any fixed constants $a''', \epsilon>0$ and for all sufficiently large $n$,
	\begin{equation}
	\int_0^\epsilon \sqrt{H_{[]}(u)} \le c\sqrt{n}\epsilon^2.
	\end{equation}
\end{lemma}
\begin{proof}
	Let $a_n = a' n^a C_n^{a''} I_n$, so $H_{[]}(u) \le \log[(a_n/u )^{I_n}] = I_n \log(a_n/u)$. Taking the square root and integrating each side, we have:
	\begin{align}
	\int_0^\epsilon \sqrt{H_{[]}(u)}~du 
	&= \int_0^\epsilon \sqrt{I_n \log(a_n/u)} ~du
	\\
	&= \sqrt{I_n/2} \int_0^\epsilon  \sqrt{2 \log(a_n/u)} ~du
	\\
	&= \sqrt{I_n/2} \int_0^\epsilon  z ~du,
	\end{align}
	where we define the substitution $z:=\sqrt{2 \log(a_n/u)}$. Then: 
	\begin{align}
	&dv = \frac{1}{2}(2\log(a_n/u))^{-1/2} (2) \frac{(-a_n/u^2)}{a_n/u}~dz = - z^{-1} u^{-1}~du
	\\
	\implies & du = -z u ~dz = -a_n z u/a_n ~dz = -a_n z \exp\left(-\frac{1}{2} \underbrace{2 \log(a_n/u)}_{z^2} \right) ~dz = -a_n z \exp(-z^2/2) ~dz.
	\end{align}
	Thus:
	\begin{align}
	\int_0^\epsilon \sqrt{H_{[]}(u)}~du 
	&\le -\sqrt{I_n/2} \int_\infty^{z_\epsilon} a_n z^2 \exp(-v^2/2) ~dz
	\\
	&= a_n \sqrt{I_n/2} \int_{z_\epsilon}^\infty z^2 \exp(-v^2/2) ~dz
	\end{align}
	where we define $z_\epsilon=\sqrt{2 \log(a_n/\epsilon)}$. Next, integrate by parts (using $u=z$ and $dv = z\exp(-z^2/2)~dz$), giving:
	\begin{align}
	\int_0^\epsilon \sqrt{H_{[]}(u)}~du 
	&= a_n \sqrt{I_n/2} \left[-z \left.\exp(-z^2/2)\right\rvert_{z_\epsilon}^\infty 
	+\int_{z_\epsilon}^\infty \exp(-z^2/2)~dz \right]
	\\
	&= a_n \sqrt{I_n/2} \left[z_\epsilon \exp(-z_\epsilon^2/2)
	+\sqrt{2\pi}\int_{z_\epsilon}^\infty\frac{1}{\sqrt{2\pi}} \exp(-z^2/2)~dz \right]
	\\
	&\le a_n \sqrt{I_n/2} \left[z_\epsilon \exp(-z_\epsilon^2/2)
	+\sqrt{2\pi} \frac{\phi(z_\epsilon)}{z_\epsilon} \right]
	&\pushright{\text{Mill's Ratio}}
	\\
	&= a_n \sqrt{I_n/2} z_\epsilon \left[\exp(-z_\epsilon^2/2)
	+\sqrt{2\pi} \frac{\frac{1}{\sqrt{2\pi}}\exp(-z_\epsilon^2/2)}{z_\epsilon^2} \right]
	\\
	&= a_n \sqrt{I_n/2} z_\epsilon \exp(-z_\epsilon^2/2) \left[1
	+ \frac{1}{z_\epsilon^2} \right]
	\\
	&= a_n \sqrt{I_n/2} z_\epsilon \underbrace{\exp(-z_\epsilon^2/2)}_{\epsilon/a_n} \left[1
	+ \frac{1}{z_\epsilon^2} \right]
	\\
	&= \epsilon \sqrt{I_n/2} z_\epsilon \left[1
	+ \frac{1}{z_\epsilon^2} \right].
	\end{align}
	Since $a_n\to \infty$ as $n\to \infty$, we have $z_\epsilon^2 = 2\log(a_n/\epsilon) \to \infty$ as well, so $[1 + 1/z_\epsilon^2] \le 2$ for large $n$. Continuing:
	\begin{align}
	\int_0^\epsilon \sqrt{H_{[]}(u)}~du 
	&\le \epsilon \sqrt{I_n/2} z_\epsilon
	\\
	&= \epsilon \sqrt{I_n/2} \sqrt{2 \log(a_n/\epsilon)}
	\\
	&= \epsilon \sqrt{I_n} \sqrt{\log(a_n/\epsilon)}
	\\
	&\le \epsilon \sqrt{I_n} \sqrt{\log(a' n^a C_n^{a''} I_n /\epsilon)}
	\\
	&\le \epsilon \sqrt{I_n} \sqrt{\log(a') + a \log(n) + a''\log(C_n) + \log(I_n) - \log(\epsilon)}
	\\
	&\le \epsilon \sqrt{(D+1)n^a+1} \sqrt{\log(a') + a \log(n) + a''n^{b-a} + \log((D+1)n^a+1) - \log(\epsilon)}
	\end{align}
	where we plug in $I_n = (D+1)K_n + 1 \le (D+1)n^a+1$ and $C_n = \exp(n^{b-a})$.
	
	Since $0<a<b<1$, there exists a $\gamma$ such that $a<\gamma<b$ and $b-a < 1 \gamma$. This follows from the fact that since $0<a<b<1$, there must exist a $\delta>0$ such that $a+\delta < b$ and $b+\delta <1$. Now let $\gamma = a \delta$ to see that $b-a=b + \delta - (a + \delta) < 1 - (a+ \delta) = 1 - \gamma$. Multiplying by $1/\sqrt{n} = \sqrt{n^{-\gamma}}\sqrt{n^{-(1-\gamma)}}$ on each side:
	\begin{align}
	\frac{1}{\sqrt{n}} \int_0^\epsilon \sqrt{H_{[]}(u)}~du 
	&\le \epsilon 
	\sqrt{n^{-\gamma}}\sqrt{(D+1)n^a+1} 
	\\&\quad
	\sqrt{n^{-(1-\gamma)}}
	\sqrt{\log(a'/\epsilon) + a \log(n) + a''n^{b-a} + \log((D+1)n^a+1)}
	\\
	&= \epsilon \sqrt{(D+1)n^{-(\gamma-a)}+n^{-\gamma}} 
	\\&\quad
	\sqrt{
		n^{-(1-\gamma)} \log(a'\epsilon)
		+ a n^{-(1-\gamma)} \log(n) 
		+ a''n^{-{((1-\gamma) - (b-a))}} 
		+ n^{-(1-\gamma)} \log((D+1)n^a+1) 
	}
	\\&\quad \to\infty\ \text{as}\ n\to\infty
	\end{align}
	since each of $\gamma$, $1-\gamma$, $\gamma-a$, and $(1-\gamma) - (b-a)$ are positive.
	Thus, for any $a''',\epsilon>0$
	\begin{equation}
	\frac{1}{\sqrt{n}}\int_0^\epsilon \sqrt{H_{[]}(u)}~du \le a''' \epsilon^2
	\end{equation}
	
\end{proof}

\begin{lemma} (Adaptation of Lemma 2 in \citep{lee_consistency_2000} (same statement but particularized for RBFNs)) \label{lemma:adapt_lemma_2_lee}
	Define the ratio of joint likelihoods between the inferred density and the true density as
	\begin{equation}
	R_n(f) = \prod_{i=1}^n \frac{f(x_i, y_i)}{f_0(x_i, y_i)}.
	\end{equation}
	Under the assumptions of Lemma \ref{lemma:adapt_lemma_1_lee}, 
	\begin{equation}
	\sup_{f\in A_\epsilon^c \cap \mathcal{F}_n} R_n(f) \le 4\exp(-a_2 n \epsilon^2)
	\end{equation}
	almost surely for sufficiently large $n$,
	where $a_2$ is the constant from Theorem \ref{thm:4}.
\end{lemma}
\begin{proof}
	
	Much of this proof is reproduced exactly as in Lemma 2 in \citep{lee_consistency_2000}, with only a few adaptations that we mention along the way. We first bound the Hellinger bracketing entropy using Theorem \ref{thm:3} and then use Lemma \ref{lemma:adapt_lemma_1_lee} to show the conditions of Theorem~\ref{thm:4}. 
	
	Since we are interested in computing the Hellinger bracketing entropy for neural networks, we need to use the $L_2$ norm on the square roots of the density function, $f$. Later, we compute the $L_\infty$ covering number of the parameter space, so here $d^*=L_\infty$. We would like to apply Theorem~\ref{thm:3} particularized for the $L_2$ norm, i.e., $|\sqrt{f_t(x,y)} - \sqrt{f_s(x,y)}| \le d^*(s,y)F(x,y)$ for some $F$ then $N_{[]}(2\epsilon \left\| F\right\|_2, \mathcal{F}^*, \left\| \cdot \right\|_2 ) \leq N(\epsilon,\mathcal{F}_n,d^*)$.
	To show that the condition holds true, apply the Fundamental Theorem of Integral Calculus. For particular vectors $s$ and $t$, let $g(u) = \sqrt{f_{(1-u)s + ut }(x,y)}$. Let $v_i = (1-u)s_i + u t_i$ and denote the space of $\theta$ by $\Theta_i$.
	\begin{align}
	|\sqrt{f_t(x,y)} - \sqrt{f_s(x,y)}| &= \int_0^1 \frac{g}{du}~du
	\\
	&= \int_0^1 \sum_{i=1}^I \frac{\partial g}{\partial\theta_i} \frac{\partial\theta_i}{\partial u}~du
	\\
	&= \sum_{i=1}^I (t_i - s_i) \int_0^1 \frac{\partial g}{\partial\theta_i} ~du
	\\
	&\le \sum_{i=1}^I \sup_i |t_i - s_i| \int_0^1 \sup_{\theta_i \in \Theta_i} \left\lvert \frac{\partial g}{\partial\theta_i}\right\rvert ~du
	\\
	&= \sup_i |t_i - s_i| \sum_{i=1}^I  \sup_{\theta_i \in \Theta_i} \left\rvert \frac{\partial g}{\partial\theta_i} \right\rvert  \int_0^1~du
	\\
	&\le \sup_i |t_i - s_i| I \sup_i \left[ \sup_{\theta_i \in \Theta_i} \left\rvert \frac{\partial g}{\partial\theta_i} \right\rvert \right]
	\\
	&= \left\| t-s \right\|_\infty F(x,y)
	\end{align}
	where $F(x,y) = I \sup_i[\sup_{\theta_i \in \Theta_i}|\partial g / \partial \theta_i| ]$. Here $\partial g / \partial \theta_i$ is the partial derivative of $\sqrt{f}$ with respect to the $i$th parameter. Recall that $f(x,y)=f(y\mid x) f(x)$, where $f(x)=1$ since $X \sim U[0,1]$ and $f(y\mid x)$ is normal with mean determined by the neural network and variance 1. 
	
	So far, this proof follows Lemma 2 in \citep{lee_consistency_2000} exactly. Now we make a slight modification for an RBFN model. By Lemma \ref{lemma:gradient_bound}, $|\partial g /\partial \theta_i| \le (8\pi e^2)^{-1/4} 2 n^a C_n^3 =  n^a C_n^3 / 2$, where $a':= 4(8\pi e^2)^{-1/4}$. Then set $F(x,y) = a' n^a C_n^3 I/2$, so $||F||_2 = a' n^a C_n^3 I/2$. Applying Theorem \ref{thm:3} to bound the bracketing number by the covering number we have:
	\begin{align}
	N_{[]}(u, \mathcal{F}^*, ||\cdot||_2) 
	&= N_{[]}\left(2 \left(\frac{u}{2||F||_2} \right) ||F||_2, \mathcal{F}^*, ||\cdot||_2\right)
	\\
	&\le N\left(\frac{u}{2||F||_2} , \mathcal{F}^*, ||\cdot||_2\right)
	\end{align}
	Notice that the covering number of $\mathcal{F}_n$ is clearly less than $((2C_n)/(2\epsilon) + 1)^I$. So, for any $\eta>0$, we have:
	\begin{equation}
	N\left(\eta , \mathcal{F}^*, L_\infty \right) \le \left(\frac{2C_n}{2\eta} + 1 \right)^I
	= \left(\frac{C_n + \eta}{\eta} \right)^I \le \left(\frac{C_n+1}{\eta} \right)^I.
	\end{equation}
	Therefore, 
	\begin{align}
	N_{[]}(u, \mathcal{F}^*, ||\cdot||_2) 
	&\le \left(\frac{C_n+1}{\frac{u}{2||F||_2}  } \right)^{I}
	\\
	&= \left(\frac{2||F||_2(C_n+1)}{u} \right)^{I}
	\\
	&= \left(\frac{a' n^a C_n^3 I_n(C_n+1)}{u} \right)^{I}
	\\
	&= \left(\frac{a' n^a \tilde{C}_n^4 I_n}{u} \right)^{I}
	\end{align}
	where $\tilde{C}_n = C_n+1$. For notational convenience, we drop $ \mathcal{F}^*$ and $||\cdot||_2$ going forward. Taking the logarithm:
	\begin{equation}
	H_{[]}(u) \le \log[(a' n^a C_n^{a''} I_n/u )^{I}].
	\end{equation}
	The bound above holds for a fixed network size, but we can now let $K_n$ grow such that $K_n \leq n^a$ for any $0<a<1$.
	Thus by Lemma \ref{lemma:adapt_lemma_1_lee}, we have:
	\begin{equation}
	\frac{1}{\sqrt{n}}\int_0^\epsilon \sqrt{H_{[]}(u)}~du \le a''' \epsilon^2,
	\end{equation}
	which shows the conditions of Lemma \ref{lemma:adapt_lemma_1_lee}. Therefore, we have that for any $a''',\epsilon>0$,
	\begin{equation}
	\int_0^\epsilon \sqrt{H_{[]}(u)}~du \le a''' \sqrt{n} \epsilon^2,
	\end{equation}
	With an eye on applying Theorem \ref{thm:4}, notice that $\int_{\epsilon^2/2^8}^\epsilon \sqrt{H_{[]}(u)}~du < \int_0^\epsilon \sqrt{H_{[]}(u)}~du$. Substituting $\sqrt{2}\epsilon$ for $\epsilon$, we get
	\begin{equation}
	\int_{\epsilon^2/2^8}^{\sqrt\epsilon} \sqrt{H_{[]}(u)}~du \le 2 a''' \sqrt{n} \epsilon^2,
	\end{equation}
	letting $a_3=1$ and $a_4 := 2a'''$, where $a_3$ and $a_4$ are the constants required by Theorem \ref{thm:4}. This gives the necessary conditions for Theorem \ref{thm:4}, which implies that
	\begin{equation}
	P^{*} \left( \sup_{f\in A_\epsilon^c \cap \mathcal{F}_n} R_n(f) \ge \exp(-a_1 n \epsilon^2) \right) \le 4\exp(-a_2 n \epsilon^2).
	\end{equation}
	Now apply the first Borel-Cantelli Lemma to get the desired result. 
\end{proof}

\subsubsection{Main theorems}

The following theorem is proved by \cite{lee_consistency_2000} for single-layer feedforward networks with a logistic activation and Gaussian priors. With a few modifications to the proof as described below, it can be applied to RBFNs. Here, the number of units is allowed to grow with the number of observations but it is not inferred from the data. We call this a sieves approach. 

\begin{theorem} (Consistency when width grows with data (sieves approach)) \citep{lee_consistency_2000} \label{thm:lee_consistency_2000_sieves}
	Suppose the following conditions hold:
	\begin{enumerate}[(i)]
		\item There exists an $r>0$ and an $N_1 \in \mathbb{N}$ such that $\forall n\ge N_1$, $\pi_n\left(\mathcal{F}_n^c \right) < \exp(-nr)$.
		
		\item For all $\gamma>0$ and $\nu>0$, there exists an $N_2 \in \mathbb{N}$ such that $\forall n\ge N_2$, $\pi_n\left(K_\gamma \right) \ge \exp(-n\nu)$.
	\end{enumerate}
	Then $\forall \epsilon>0$, the posterior is asymptotically consistent for $f_0$ over Hellinger neighborhoods, i.e.:
	\begin{equation}
	P(A_\epsilon \mid (x_1,y_1), \dotsc, (x_n,y_n)) \overset{p}{\to} 1.
	\end{equation}
\end{theorem}
\begin{proof}
	\cite{lee_consistency_2000} proves this result for single-layer feedforward networks with a logistic activation and Gaussian priors (Theorem 1 in their paper). Their proof relies on their Lemmas 3 and 5. Their Lemma 5 needs no adaptation for RBFNs but their Lemma 3 depends on their Lemma 2, which does need adaptation for RBFNs. Above we proved their Lemma 2 for RBFNs, which we call Lemma \ref{lemma:adapt_lemma_1_lee}. Thus their Lemma 3 holds, so their Theorem 1 holds, which gives the results of this theorem.
\end{proof}

\cite{lee_consistency_2000} shows that Hellinger consistency gives asymptotic consistency. 

\begin{corollary} (Hellinger consistency gives asymptotic consistency for sieves prior) \citep{lee_consistency_2000} \label{thm:lee_consistency_2000_corollary1}
	Under the conditions of Theorem \ref{thm:lee_consistency_2000_sieves}, $\hat{g}_n$ is asymptotically consistent for  $g_0$, i.e.:
	\begin{equation}
	\int (\hat{g}_n(x) - g_0(x))^2 dx \overset{p}{\to} 0.
	\end{equation}
\end{corollary} 

The following is an extension of Theorem \ref{thm:lee_consistency_2000_sieves} to when there is a prior over the number of units. The proof in \citep{lee_consistency_2000} assumes a feedforward network with a logistic activation and Gaussian priors, but these assumptions are not used beyond their use in applying Theorem \ref{thm:lee_consistency_2000_sieves}. Since we adapt Theorem \ref{thm:lee_consistency_2000_sieves} to our model, the proof of the following Theorem \ref{thm:lee_consistency_2000_poisson} needs no additional adaptation. 

\begin{theorem} (Consistency for prior on width) \citep{lee_consistency_2000} \label{thm:lee_consistency_2000_poisson}
	Suppose the following conditions hold:
	\begin{enumerate}[(i)]
		\item For each $i=1,2,\dotsc$ there exists a real number $r_i>0$ and an integer $N_i>0$ such that $\forall n\ge N_i$, $\pi_i \left(\mathcal{F}_n^c \right) < \exp(-r_i n)$.
		\item For all $\gamma, \nu >0$ there exists an integer $I>0$ such that for any $i>I$ there exists an integer $M_i>0$ such that for all $n\ge M_i$, $\pi_i(K_\gamma) \ge \exp(-\nu n)$.
		\item $B_n$ is a bound that grows with $n$ such that for all $r>0$ there exists a real number $q>1$ and an integer $N>0$ such that for all $n\ge N$, $\sum_{i=B_n}^\infty \lambda_i < \exp(-r n^q)$.
		\item For all $i$, $\lambda_i>0$.
	\end{enumerate}
	Then $\forall \epsilon>0$, the posterior is asymptotically consistent for $f_0$ over Hellinger neighborhoods, i.e.:
	\begin{equation}
	P(A_\epsilon \mid (x_1,y_1), \dotsc, (x_n,y_n)) \overset{p}{\to} 1.
	\end{equation}
\end{theorem}

\begin{corollary} (Hellinger consistency gives asymptotic consistency for prior on width). \label{thm:corollary2} 
	Under the conditions of Theorem \ref{thm:lee_consistency_2000_poisson}, $\hat{g}_n$ is asymptotically consistent for  $g_0$, i.e.:
	\begin{equation}
	\int (\hat{g}_n(x) - g_0(x))^2 dx \overset{p}{\to} 0.
	\end{equation}
\end{corollary} 
\begin{proof}
	The conditions of Theorem \ref{thm:lee_consistency_2000_poisson} imply the conditions of Theorem \ref{thm:lee_consistency_2000_sieves}, so then Corollary \ref{thm:lee_consistency_2000_corollary1} must hold.
\end{proof}

\subsection{Consistency for PoRB-NET} \label{sec:consistency_porbnet}

\subsubsection{Supporting results}

\begin{theorem} (RBFNs are universal function approximators) \cite{park_universal_1991} 	\label{thm:park_universal_1991} Define $S_\phi$ as the set of all functions of the form:
	\begin{equation}
	\text{RBFN}_\phi(x; \theta) = \sum_{k=1}^K w_k \phi \left(\lambda(x-c_k) \right),\label{eq:rbfn_phi}
	\end{equation}
	where $\lambda>0$, $w_k\in\mathbb{R}$, $c_k \in \mathbb{R}^D$ and $\theta = \{\{w_k\}_{k=1}^K, \{c_k\}_{k=1}^K, \lambda \}$ is the collection of network parameters.
	If $\phi:\mathbb{R}^d \to \mathbb{R}$ is an integrable bounded function such that $\phi$ is continuous almost everywhere and $\int_{\mathbb{R}^d} \phi(z)~dz \neq 0$, then the family $S_\phi$ is dense in $L_p(\mathbb{R}^d)$ for every $p\in [1,\infty)$.
\end{theorem}
In our case, $\phi(z)=\exp(-z^2)$, which clearly satisfies the conditions of Theorem \ref{thm:park_universal_1991}. We will denote $\text{RBFN}(x;\theta)$ the expression in Equation~\eqref{eq:rbfn_phi} particularized for the squared exponential $\phi$ function.

\begin{lemma}(Bound on network gradients)\label{lemma:gradient_bound}
	\begin{equation}
	\frac{\partial \sqrt{f(x,y; \theta)}}{\partial \theta_i} \le (8\pi e^2)^{-1/4} \frac{\partial \text{RBFN}(x;\theta)}{\partial \theta} = (8\pi e^2)^{-1/4} 2 n^a C_n^3
	\end{equation}
\end{lemma}
\begin{proof}
	Applying the chain rule we have:
	\begin{align}
	\left\lvert \frac{\partial \sqrt{f(x,y; \theta)}}{\partial \theta_i} \right\rvert
	&= \frac{1}{2} \left(f(x,y; \theta) \right)^{-1/2} \frac{\partial f(x,y; \theta)}{\partial \theta_i} 
	\\
	&= \frac{1}{2} (2\pi)^{-1/4} \exp\left(-\frac{1}{4} (y-\text{RBFN}(x;\theta))^2 \right) \left\lvert y-\text{RBFN}(x;\theta) \right\rvert 
	\left\lvert \frac{\partial RBFN(x;\theta_i)}{\partial \theta_i} \right\rvert \label{eq:chain_rule}
	\end{align}
	First we show that we can bound the middle terms by:
	\begin{align}
	\exp\left(-\frac{1}{4} (y-\text{RBFN}(x;\theta))^2 \right) \left\lvert y-\text{RBFN}(x;\theta)\right\rvert 
	&\le \exp(-1/2)2^{1/2} \label{eq:bound_middle_terms}
	\end{align}
	To see this, rewrite the left-hand-side of Equation \ref{eq:bound_middle_terms} as $s(z) := \exp(-(1/4)z^2)|z|$, where $z=y-\text{RBFN}(x;\theta)$. Taking the derivative we have:
	\begin{align}
	\frac{\partial s(z)}{\partial z} &=
	\begin{cases}
	-\frac{1}{2} z^2\exp(-\frac{1}{4} z^2) + \exp(-\frac{1}{4} z^2)  & z\ge 0 \\
	\frac{1}{2} z^2\exp(-\frac{1}{4} z^2) - \exp(-\frac{1}{4} z^2)  & z< 0
	\end{cases}
	\\
	&=
	\begin{cases}
	\exp(-\frac{1}{4} z^2)(-\frac{1}{2}z^2 + 1)  & z\ge 0 \\
	\exp(-\frac{1}{4} z^2)(\frac{1}{2}z^2 - 1) & z< 0
	\end{cases}
	\end{align}
	Setting to zero, we must have that $\frac{1}{2}z^2 = 1 \implies z=\sqrt{2}$. Thus, $a(z) \le \exp(-1/2)2^{1/2} $, as in Equation \ref{eq:bound_middle_terms}.
	
	Next, consider the derivatives of the radial basis function network:
	\begin{align}
	\left\lvert \frac{\partial RBFN(x;\theta_i)}{\partial b} \right\rvert = 1
	\end{align}
	\begin{align}
	\left\lvert \frac{\partial RBFN(x;\theta_i)}{\partial w_k} \right\rvert = \exp(-\frac{1}{2}\lambda^2(x-c_k)^2 ) \le 1
	\end{align}
	\begin{align}
	\left\lvert \frac{\partial RBFN(x;\theta_i)}{\partial w_k} \right\rvert &= |w_k| \exp(-\frac{1}{2}\lambda^2(x-c_k)^2 ) \lambda^2 |x-c| \\
	&\le |w_k| \lambda^2 (|c_k|+1) \\
	&\le C_n^2 (C_n + 1) \\
	&\le C_n^3 + C_n^2 \\
	&\le 2 C_n^3
	\end{align}
	since $ C_n^2 = \exp(2n^{b-a}) < \exp(3n^{b-a}) = C_n^3$
	\begin{align}
	\left\lvert \frac{\partial RBFN(x;\theta_i)}{\partial w_k} \right\rvert 
	&= \frac{1}{2} \left \lvert \sum_{k=1}^{K_n} w_k \exp(-\frac{1}{2}\lambda^2(x-c_k)^2 )  (x-c)^2 \right \rvert \\
	&= \frac{1}{2}  \sum_{k=1}^{K_n} \left \lvert w_k \exp(-\frac{1}{2}\lambda^2(x-c_k)^2 )  (x-c)^2 \right \rvert \\
	&\le \frac{1}{2} \sum_{k=1}^{K_n} |w_k|  (|c|+1)^2 \\
	&\le \frac{1}{2} \sum_{k=1}^{n^a} C_n  (C_n + 1)^2 \\
	&= \frac{1}{2} n^a C_n  (C_n + 1)^2 \\
	&= \frac{1}{2} n^a C_n  (C_n^2 + 2C_n + 1) \\
	&= \frac{1}{2} n^a (C_n^3 + 2C_n^2 + C_n) \\
	&\le \frac{1}{2} n^a (C_n^3 + 2C_n^3 + C_n^3) \\
	&\le 2 n^a C_n^3
	\end{align}
	Plugging everything in to Equation \ref{eq:chain_rule} we have the desired inequality. 
\end{proof}

\begin{lemma} (Bounding sum of exponentially bounded terms). \label{lemma:bound_sum_of_exponentials}
	For two sequences $\{a_n\}_{n=1}^\infty$ and $\{b_n\}_{n=1}^\infty$ suppose there exists real numbers $r_a>0$ and $r_b>0$ as well as integers $N_a>0$ and $N_b>0$ such that $a_n \le \exp(-r_a n)$ for all $n\ge N_a$ and $b_n \le \exp(-r_b n)$ for all $n\ge N_b$. Then there exists a real number $r>0$ and an integer $N>0$ such that $a_n + b_n \le \exp(-r n)$ for all $n\ge N$. 
\end{lemma}
\begin{proof}
	Set $\tilde{r}=\min\{r_a, r_b\}$ and $\tilde{N}=\max\{N_a, N_b\}$. Then we have:
	\begin{align}
	a_n &\le \exp(-r_a n), \quad \forall n\ge \tilde{N} \ge N_a \\
	&\le \exp(-\tilde{r} n,) \quad \forall n\ge \tilde{N}
	\end{align}.
	Similarly, $b_n \le \exp(-\tilde{r} n)$, $\forall n\ge \tilde{N}$. Thus we have $a_n + b_n \le 2\exp(-\tilde{r} n)$, $\forall n\ge \tilde{N}$.
	
	Now set $N = \max\{\lceil \frac{\log 2}{\tilde{r}} \rceil +1, \tilde{N} \} $ and $r = \tilde{r} - \frac{\log 2}{N}$. Notice $r>0$, since $N \ge \lceil \frac{\log 2}{r} \rceil +1 >  \frac{\log 2}{r} $ implies $r = \tilde{r} - \frac{\log 2}{N} > \tilde{r} - \log 2 \frac{\tilde{r}}{\log 2} = 0$.
	It follows that $2\exp(-r n) \le \exp(-r n)$, $\forall n\ge N$, since:
	\begin{align}
	2\exp(-\tilde{r} n) &\le \exp(-r n) \\
	\iff \log 2 - \tilde{r}n &\le -rn \\
	\iff \log 2 - \tilde{r}n &\le - \left( \tilde{r} - \frac{\log 2}{N} \right)n \\
	\iff \log 2 - \tilde{r}n &\le - \tilde{r}n + \frac{n\log 2}{N} \\
	\iff N &\le n
	\end{align}
\end{proof}

\begin{lemma} (Useful equality)
	For all $\delta \le 1$ and $x\in[0,1]$, if $|\tilde{c}- c| \le \delta$ and $|\tilde{\lambda}- \lambda| \le \delta$, then 
	there exists a constant $\xi$ such that $|\xi| \le A(|c|, \lambda) \delta$ and:
	\begin{equation}
	\tilde{\lambda}^2 (x-\tilde{c})^2 =
	\lambda^2 (x-c)^2 + \xi,
	\end{equation}
	where $A(|c|, \lambda) = 2 \lambda (|c|+1) (\lambda + |c| + 2) + (\lambda + |c| + 2)^2$
	\label{lemma:1}
\end{lemma}
\begin{proof}
	Since $|\tilde{c}- c| \le \delta$ and $|\tilde{\lambda}- \lambda| \le \delta$ there exists constants $\xi_1$ and $\xi_2$, where $|\xi_1|\le \delta$ and $|\xi_2|\le \delta$, such that $\tilde{c} = c+\xi_1$ and $\tilde{\lambda} = \lambda+\xi_2$
	
	Plugging $\tilde{c} = c+\xi$ and $\tilde{\lambda} = \lambda+\xi_2$ into the left-hand-side of the desired inequality:
	\begin{align}
	\tilde{\lambda}(x-\tilde{c}) &= (\lambda+\xi_2)(x - c - \xi_1) \\
	&= \lambda(x-c) + \underbrace{(-\lambda \xi_1) + \xi_2(x-c) - \xi_1 \xi_2}_{:=\xi_3}
	\end{align}
	Notice:
	\begin{align}
	|\xi_3| &= | (-\lambda \xi_1)  + \xi_2(x-c) - \xi_1 \xi_2| \\
	&\le \lambda |\xi_1| + |\xi_2| |x-c| + |\xi_1| |\xi_2| \\
	&\le \lambda \delta + \delta (|c| + 1) + \delta^2 \label{eq:use_x_bound}\\
	&\le (\lambda + |c| + 2)\delta \label{eq:use_del2}
	\end{align}
	In Equation \ref{eq:use_x_bound} we use $|x-c|\le (|c|+1)$, which follows since we assume $x\in[0,1]$, as well as $\xi_1\le\delta$ and $\xi_2\le \delta$. In Equation \ref{eq:use_del2} we use $\delta^2 \le \delta$, which follows since we assume $\delta \le 1$.
	Squaring the left-hand-side of the desired inequality:
	\begin{align}
	\tilde{\lambda}^2(x-\tilde{c})^2 &= (\tilde{\lambda}(x-\tilde{c}))^2 \\
	&= \left(\lambda(x-c) + \xi_3 \right)^2 \\
	&= \lambda^2(x-c)^2 + \underbrace{2\lambda(x-c)\xi_3 + \xi_3^2}_{:=\xi_4}
	\end{align}
	Notice:
	\begin{align}
	|\xi_4| &= |2\lambda(x-c)\xi_3 + \xi_3^2| \\
	&\le 2 \lambda |x-c| |\xi_3| + |\xi_3^2| \\
	&\le 2 \lambda (|c|+1) (\lambda + |c| + 2)\delta + (\lambda + |c| + 2)^2\delta^2 \label{eq:use_xi_3} \\
	&\le \underbrace{\left( 2 \lambda (|c|+1) (\lambda + |c| + 2) + (\lambda + |c| + 2)^2 \right)}_{:=A(|c|, \lambda)} \delta \label{eq:use_del2_again}
	\end{align}
	In Equation \ref{eq:use_xi_3} we use $|\xi_3| \le (\lambda + |c| + 2)\delta$ and $|x-c|\le (|c|+1)$ again and Equation \ref{eq:use_del2_again} we use $\delta^2 \le \delta$. 
	This proves the desired inequality for $\xi := \xi_4$.
\end{proof}

\begin{lemma} (Proximity in parameter space leads to proximity in function space). \label{lemma:6}	
	Let $g$ be an RBFN with $K$ nodes and parameters $(\theta_1,\dotsc,\theta_I)$ and let $\tilde{g}_n$ be an RBFN with $\tilde{K}_n$ nodes and parameters $(\tilde{\theta}_1,\dotsc,\tilde{\theta}_{\tilde{I}(n)})$, where $\tilde{K}_n$ grows with $n$. Define $\theta_i=0$ for $i>\mathcal{I}$,  $\tilde{\theta}_i = 0$ for $i>\tilde{i}$, and $M_\delta$, for any $\delta>0$, as the set of all networks $\tilde{g}$ that are close in parameter space to $g$:
	\begin{equation}
	M_\delta(g):=\{ \tilde{g}_n \mid |\tilde{\theta_i} -\theta_i |, i=1,\dotsc \}
	\end{equation}
	Then for any $\tilde{g} \in M_\delta$ and sufficiently large $n$,
	\begin{equation}
	\sup_{x\in\mathcal{X}} \left(\tilde{g}(x) - g(x) \right)^2
	\le \left( 3 \tilde{K}_n\right)^2 \delta^2
	\end{equation}
\end{lemma}
\begin{proof}
	\begin{align}
	\sup_{x\in\mathcal{X}}& \left(\tilde{g}(x) - g(x) \right)^2
	\\
	&=
	\sup_{x\in\mathcal{X}} 
	\left( 
	\tilde{b} + \sum_{k=1}^{\tilde{K}_n} \tilde{w}_k \exp(-\tilde{\lambda}^2 (x-\tilde{c}_k)^2) -
	b - \sum_{k=1}^{K} w_k \exp(-\lambda^2 (x-c_k)^2)
	\right)^2
	\\
	&=
	\sup_{x\in\mathcal{X}} 
	\left( 
	(\tilde{b} - b) + 
	\left( 
	\sum_{k=1}^{\tilde{K}_n} \tilde{w}_k \exp(-\tilde{\lambda}^2 (x-\tilde{c}_k)^2) -
	\sum_{k=1}^{K} w_k \exp(-\lambda^2 (x-c_k)^2)
	\right)
	\right)^2
	\\
	&=
	\sup_{x\in\mathcal{X}} 
	\left( 
	(\tilde{b} - b) + 
	\left( 
	\sum_{k=1}^{\tilde{K}_n^*} \tilde{w}_k \exp(-\tilde{\lambda}^2 (x-\tilde{c}_k)^2) -
	w_k \exp(-\lambda^2 (x-c_k)^2)
	\right)
	\right)^2
	\\
	&\le
	\sup_{x\in\mathcal{X}} 
	\left( 
	|\tilde{b} - b| + 
	\left \lvert
	\sum_{k=1}^{\tilde{K}_n^*} \tilde{w}_k \exp(-\tilde{\lambda}^2 (x-\tilde{c}_k)^2) -
	w_k \exp(-\lambda^2 (x-c_k)^2)
	\right \rvert
	\right)^2
	\\
	&=
	\sup_{x\in\mathcal{X}} \left[
	|\tilde{b} - b|^2 + 
	2|\tilde{b} - b| \left\lvert
	\sum_{k=1}^{\tilde{K}_n^*} \tilde{w}_k \exp(-\tilde{\lambda}^2 (x-\tilde{c}_k)^2) -
	w_k \exp(-\lambda^2 (x-c_k)^2)
	\right \rvert \right.
	\\
	&\quad+\left. \left\lvert
	\sum_{k=1}^{\tilde{K}_n^*} \tilde{w}_k \exp(-\tilde{\lambda}^2 (x-\tilde{c}_k)^2) -
	w_k \exp(-\lambda^2 (x-c_k)^2)
	\right \rvert^2 \right]
	\\
	&\le
	|\tilde{b} - b|^2 + 
	2|\tilde{b} - b| \sup_{x\in\mathcal{X}} \left\lvert
	\sum_{k=1}^{\tilde{K}_n^*} \tilde{w}_k \exp(-\tilde{\lambda}^2 (x-\tilde{c}_k)^2) -
	w_k \exp(-\lambda^2 (x-c_k)^2)
	\right \rvert 
	\\
	&\quad+\sup_{x\in\mathcal{X}}  \left\lvert
	\sum_{k=1}^{\tilde{K}_n^*} \tilde{w}_k \exp(-\tilde{\lambda}^2 (x-\tilde{c}_k)^2) -
	w_k \exp(-\lambda^2 (x-c_k)^2)
	\right \rvert^2
	\\
	&=
	|\tilde{b} - b|^2 + 
	2|\tilde{b} - b| \sup_{x\in\mathcal{X}} \left\lvert
	\sum_{k=1}^{\tilde{K}_n^*} \tilde{w}_k \exp(-\tilde{\lambda}^2 (x-\tilde{c}_k)^2) -
	w_k \exp(-\lambda^2 (x-c_k)^2)
	\right \rvert 
	\\
	&\quad+\left( \sup_{x\in\mathcal{X}}  \left\lvert
	\sum_{k=1}^{\tilde{K}_n^*} \tilde{w}_k \exp(-\tilde{\lambda}^2 (x-\tilde{c}_k)^2) -
	w_k \exp(-\lambda^2 (x-c_k)^2)
	\right \rvert \right)^2
	\\
	&\le
	|\tilde{b} - b|^2 + 
	2|\tilde{b} - b| \sup_{x\in\mathcal{X}}
	\sum_{k=1}^{\tilde{K}_n^*} \left\lvert \tilde{w}_k \exp(-\tilde{\lambda}^2 (x-\tilde{c}_k)^2) -
	w_k \exp(-\lambda^2 (x-c_k)^2)
	\right \rvert 
	\\
	&\quad+\left( \sup_{x\in\mathcal{X}} 
	\sum_{k=1}^{\tilde{K}_n^*} \left\lvert \tilde{w}_k \exp(-\tilde{\lambda}^2 (x-\tilde{c}_k)^2) -
	w_k \exp(-\lambda^2 (x-c_k)^2)
	\right \rvert \right)^2
	\\
	&\le
	|\tilde{b} - b|^2 + 
	2|\tilde{b} - b| 
	\sum_{k=1}^{\tilde{K}_n^*} \sup_{x\in\mathcal{X}} \left\lvert \tilde{w}_k \exp(-\tilde{\lambda}^2 (x-\tilde{c}_k)^2) -
	w_k \exp(-\lambda^2 (x-c_k)^2)
	\right \rvert 
	\\
	&\quad+\left( 
	\sum_{k=1}^{\tilde{K}_n^*} \sup_{x\in\mathcal{X}} \left\lvert \tilde{w}_k \exp(-\tilde{\lambda}^2 (x-\tilde{c}_k)^2) -
	w_k \exp(-\lambda^2 (x-c_k)^2)
	\right \rvert \right)^2
	\\
	&=
	|\tilde{b} - b|^2 + 
	2|\tilde{b} - b| 
	\sum_{k=1}^{\tilde{K}_n^*} \Gamma_k+\left( 
	\sum_{k=1}^{\tilde{K}_n^*} \Gamma_k \right)^2
	\\
	&\le\delta^2 + 
	2\delta 
	\sum_{k=1}^{\tilde{K}_n^*} \Gamma_k+\left( 
	\sum_{k=1}^{\tilde{K}_n^*} \Gamma_k \right)^2
	, \label{eq:end}
	\end{align}
	
	where:
	\begin{equation}
	\Gamma_k : =  \sup_{x\in\mathcal{X}} \left\lvert \tilde{w}_k \exp(-\tilde{\lambda}^2 (x-\tilde{c}_k)^2) -
	w_k \exp(-\lambda^2 (x-c_k)^2)
	\right \rvert \label{eq:gamma}
	\end{equation}
	Let $u(x)^2 := \lambda^2 (x-c_k)^2$ and $\tilde{u}(x)^2= \tilde{\lambda}^2(x-\tilde{c}_k)^2$ and pick any $x\in\mathcal{X}$.
	By Lemma \ref{lemma:1} there exists a constant $\eta$ such that $|\eta| \le A(|c|, \lambda) \delta$ and
	\begin{equation}
	\tilde{u}(x)^2 = u(x)^2 + \eta. \label{eq:apply_lemma_6}
	\end{equation}
	Now define $\xi=\sqrt{|\eta|}$ and consider two cases. 
	\begin{itemize}
		\item If $\tilde{u}(x)^2 \ge u(x)^2$, then Equation \ref{eq:apply_lemma_6} is equivalent to $\tilde{u}(x)^2 = u(x)^2 + \xi^2$. Then $\Gamma_k$ becomes:
		\begin{align}
		\Gamma_k &=  \sup_{x\in\mathcal{X}} \left\lvert 
		\tilde{w}_k \exp(-\tilde{u}^2(x)) -
		w_k \exp(-u^2(x))
		\right \rvert
		\\
		&=\left\lvert \tilde{w}_k \exp(-u^2(x) - \xi^2) -
		w_k \exp(-u^2(x)^2)
		\right \rvert
		\\
		&=  \sup_{x\in\mathcal{X}}\exp(-u(x)^2)  \left\lvert \tilde{w}_k \exp(-\xi^2) -
		w_k 
		\right \rvert \\
		&=   \left\lvert \tilde{w}_k \exp(-\xi^2) -
		w_k 
		\right \rvert \sup_{x\in\mathcal{X}}\exp(-u(x)^2) \\
		&\le  \left\lvert \tilde{w}_k \exp(-\xi^2) -
		w_k 
		\right \rvert
		\end{align}
		Since $|\tilde{w}_k - w_k| \le \delta$, there exists $\tau$, where $|\tau| \le \delta$, such that $\tilde{w}_k = w_k + \tau$. Plugging this in:
		\begin{align}
		\Gamma_k &\le \left\lvert (w_k + \tau) \exp(-\xi^2) - w_k \right \rvert \\
		&\le \left\lvert w_k(\exp(-\xi^2) - 1) + \tau \right \rvert \\
		&\le | w_k | |  \exp(-\xi^2) -1 | + |\tau| \label{eq:use_exp_inquality}\\
		&\le |w_k| \xi^2 + \delta,
		\end{align}
		where we use the result that $1-\xi^2 \le \exp(-\xi^2)$ in Equation \ref{eq:use_exp_inquality}.

		\item If $\tilde{u}(x)^2 < u(x)^2$, then Equation \ref{eq:apply_lemma_6} is equivalent to $u(x)^2 = \tilde{u}(x)^2 + \xi^2$. Then $\Gamma_k$ becomes:
		\begin{align}
		\Gamma_k &=  \sup_{x\in\mathcal{X}} \left\lvert 
		\tilde{w}_k \exp(-\tilde{u}^2(x)) -
		w_k \exp(-u^2(x))
		\right \rvert
		\\
		&=\sup_{x\in\mathcal{X}} \left\lvert \tilde{w}_k \exp(-\tilde{u}^2(x) -
		w_k \exp(-\tilde{u}^2(x) - \xi^2)
		\right \rvert
		\\
		&=  \sup_{x\in\mathcal{X}}\exp(-\tilde{u}^2(x))  \left\lvert \tilde{w}_k -
		w_k \exp(-\xi^2)
		\right \rvert \\
		&=   \left\lvert \tilde{w}_k -
		w_k \exp(-\xi^2)
		\right \rvert \sup_{x\in\mathcal{X}}\exp(-\tilde{u}^2(x)) \\
		&\le  \left\lvert \tilde{w}_k -
		w_k \exp(-\xi^2)
		\right \rvert
		\end{align}
		Using the same $\tau$ as above:
		\begin{align}
		\Gamma_k &\le \left\lvert (w_k + \tau) \exp(-\xi^2) - w_k \right \rvert \\
		&\le \left\lvert w_k(1-\exp(-\xi^2)) + \tau \right \rvert \\
		&\le | w_k | |  1-\exp(-\xi^2) | + |\tau| \\
		&= | w_k | |  \exp(-\xi^2) -1 | + |\tau| \\
		&\le |w_k| \xi^2 + \delta.
		\end{align}
		
	\end{itemize} 
	In either of the two cases, we have $\Gamma_k \le |w_k| \xi^2 + \delta$. Proceeding:
	\begin{align}
	\Gamma_k &\le |w_k| \xi^2 + \delta \\
	&\le |w_k| A(|c|, \lambda)\delta + \delta \\
	&= (|w_k| A(|c|, \lambda) + 1) \delta
	\end{align}

	Now consider 
	\begin{align}
	\sum_{k=1}^{\tilde{K}_n^*} \Gamma_k &\le \sum_{k=1}^{\tilde{K}_n} \Gamma_k &\pushright{\text{for large } n}
	\\
	&\le \delta \sum_{k=1}^{\tilde{K}_n} (|w_k| A(|c|, \lambda) + 1) \\
	&\le \delta \left( \sum_{k=1}^{\tilde{K}_n} |w_k| A(|c|, \lambda) + \tilde{K}_n \right) \\
	&\le \delta \left( \tilde{K}_n + \tilde{K}_n \right) \label{eq:large_n} \\
	&= 2 \delta \tilde{K}_n &\pushright{\text{for large } n}
	\label{eq:bound_sum_gamma}
	\end{align}
	Equation \ref{eq:large_n} follows because for $k\ge K$, $w_k=0$ by definition, so $\sum_{k=1}^{\tilde{K}_n} |w_k| A(|c|, \lambda)$ is a constant and thus less than $\tilde{K}_n$ for large $n$. 
	
	Plugging Equation \ref{eq:bound_sum_gamma} into Equation \ref{eq:end}:
	\begin{align}
	\sup_{x\in\mathcal{X}} \left(\tilde{g}(x) - g(x) \right)^2 
	&\le \delta^2 + 2(2 \delta \tilde{K}_n) + (2 \delta \tilde{K}_n)^2 \\	
	&= \left(1 + 2(2\tilde{K}_n) + (2\tilde{K}_n)^2 \right) \delta^2\\
	&= \left(1 + 2\tilde{K}_n) \right)^2 \delta^2\\
	&\le (3 \tilde{K}_n )^2  \delta^2
	\end{align}

\end{proof}

\subsubsection{Main theorems for PoRB-NET}

Recall the generative model for PoRB-NET in the case of a uniform intensity function with a Gamma prior on its level. For simplicity and w.l.o.g, we consider the case where the hyperparameter $s^2_0$ and the observation variance are fixed to 1. 

We first consider the case where the width of the network is allowed to grow with the data but is fixed in the prior. We call the estimated regression function $\hat{g}_n$, with width $K_n$ and prior $\pi_n$, where $n$ is the number of observations. The following theorem gives consistency for this model. 

Note that the following proof uses \citep{park_universal_1991} to show the existence of a neural network that approximates any square integrable function. We assume that the center parameters of this network are contained in the bounded region over which the Poisson process is defined, which can be made arbitrarily large. 

\begin{theorem} (PoRB-NET consistency with fixed width that grows with the number of observations).
	If there exists a constant $a \in (0,1)$ such that $K_n \le n^a$, and $K_n \to \infty$ as $n\to\infty$, then 
	for any square integrable ground truth regression function $g_0$, $\hat{g}_n$ is asymptotically consistent for $g$ as $n\to\infty$, i.e.
	\begin{equation}
	\int (\hat{g}_n(x) - g_0(x))^2 dx \overset{p}{\to} 0.
	\end{equation}
\end{theorem}

\begin{proof}
	
	\item
	\paragraph{Proof outline}
	\begin{itemize}
		\item Show Condition (i) of Theorem \ref{thm:lee_consistency_2000_sieves} is met
		\begin{itemize}
			\item Write prior probability of large parameters as a sum of integrals over each parameter
			\item Bound each set of parameters:
			\begin{itemize}
				\item Bound weights (as in \cite{lee_consistency_2000})
				\item Bound centers (trivial since parameter space bounded)
				\item Bound $\lambda^2$ with Chernoff bound
			\end{itemize}
			\item Bound sum using Lemma \ref{lemma:bound_sum_of_exponentials}
		\end{itemize}
		
		\item Show Condition (ii) of Theorem \ref{thm:lee_consistency_2000_sieves} is met.
		\begin{itemize}
			\item Assume true regression function $g_0$ is $L_2$
			\item Use Theorem~\ref{thm:park_universal_1991} to find an RBFN $g$ that approximates $g_0$
			\item Define $M_\delta$ as RBFNs close in parameter space to $g$
			\item Show $M_\delta \subset K_\gamma$ using Lemmas \ref{lemma:1} and \ref{lemma:6}.
			\item Show $\pi_n\left(M_\delta \right) \ge \exp(-r n)$:
			\begin{itemize}
				\item Show you can write as a product of integrals over parameters
				\item Bound each term separately:
				\begin{itemize}
					\item Bound weights as in \cite{lee_consistency_2000}
					\item Bound centers and $\lambda^2$
				\end{itemize}
			\end{itemize}
		\end{itemize}
		
	\end{itemize}

	\item
	\paragraph{Condition (i)} We want to show that there exists an $r>0$ and an $N_1 \in \mathbb{N}$ such that $\forall n\ge N_1$:
	$$
	\pi_n\left(\mathcal{F}_m^c \right) < \exp(-nr).
	$$
	
	\item
	\subparagraph{Write prior probability of large parameters as a sum of integrals over each parameter.}
	
	The prior $\pi_n$ assigns zero probability to RBFNs with anything but $K_n$ nodes, so there is no issue writing $\pi_n(\mathcal{F}_n)$ and its value is equivalent to $\pi_n(\mathcal{G}_n)$, even though $\mathcal{G}_n \subset \mathcal{F}_n$. 
	
	Notice that $\pi_n\left( \mathcal{G}_n^c \right)$ requires evaluating a multiple integral over a subset of the product space of $I_n$ parameters. 
	Notice $\mathcal{G}_n$ can be written as an intersection of sets:
	$$
	\mathcal{G}_n = \bigcap_{i=1}^{I_n} \{\text{RBFN} \in \mathcal{H}_n \mid |\theta_i| \le C_n \}.
	$$
	Therefore we have:
	\begin{align}
	\pi_n\left( \mathcal{F}_n^c \right) &= \pi_n\left( \mathcal{G}_n^c \right) \nonumber \\
	&=  \pi_n\left( \left[\bigcap_{i=1}^{I_n} \{\text{RBFN} \in \mathcal{H}_n \mid |\theta_i| \le C_n \} \right]^c \right) \nonumber\\
	&= \pi_n\left( \bigcup_{i=1}^{I_n} \{\text{RBFN} \in \mathcal{H}_n \mid |\theta_i| \le C_n \}^c \right) &\pushright{\text{De Morgan}} \nonumber\\
	&= \pi_n\left( \bigcup_{i=1}^{I_n} \{\text{RBFN} \in \mathcal{H}_n \mid |\theta_i| > C_n \} \right) \nonumber \\
	&\le \sum_{i=1}^{I_n} \pi_n\left(  \{\text{RBFN} \in \mathcal{H}_n \mid |\theta_i| > C_n \} \right) &\pushright{\text{Union bound}} \label{eq:union}.
	\end{align}
	
	Next, independence in the prior will allow us to write each term in Equation \ref{eq:union} as an integral over a single parameter. Define the following sets:
	\begin{align*}
	\mathcal{C}_i(n) &:= \Theta_i \setminus [-C_n, C_n] \\
	\mathcal{R}_i(n) &:= \Theta_1 \times \dotsc \times \Theta_{i-1} \times \mathcal{C}_i(n) \times \Theta_{i+1} \times \dotsc \times \Theta_{I_n}
	\end{align*}
	where $\Theta_i$ is the parameter space corresponding to parameter $\theta_i$ (either $\mathbb{R}$ or $\mathbb{R}^+$). Notice that because $\mathcal{R}_i(n)$ is a union of two rectangular sets (one where $\theta_i$ is less than $-C_n$ and one where $\theta_i$ is greater than $C_n$), we can apply Fubini's theorem.
	Thus, each term in Equation \ref{eq:union} can be written as:	
	\begin{align}
	\pi_n &\left(  \{\text{RBFN} \in \mathcal{H}_n \mid |\theta_i| > C_n \} \right)
	\\
	&=\int \dotsc \int_{\mathcal{R}_i(n)} \pi_n(\theta_1,\dotsc,\theta_{I_n}) d(\theta_1,\dotsc,\theta_{I_n}) \label{eq:apply_fubini}
	\\
	&=\int d\theta_1 \dotsc \int d\theta_{I_n} ~ \pi_n(\theta_1,\dotsc,\theta_{I_n})
	\\
	&=\int d\lambda^2 \int dw_1 \dotsc \int dw_{I_n^{w}}  \int dc_1 \dotsc \int dw_{I_n^{c}} ~\pi_n(\lambda^2) \prod_j \pi_n(c_j\mid \lambda^2) \prod_j \pi_n(w_j) 
	\\
	&=\left(\int d\lambda^2 ~ \pi_n(\lambda^2) \int dc_1 \dotsc \int dc_{I_n^{c}} ~ \prod_j \pi_n(c_j\mid \lambda^2) \right)
	\left(
	\int dw_1 \dotsc \int dw_{I_n^{w}}  ~ \prod_j \pi_n(w_j) 
	\right)
	\\
	&=\left(\int d\lambda^2 ~ \pi_n(\lambda^2) \prod_j \int dc_j ~ \pi_n(c_j\mid \lambda^2) \right)
	\left(
	\prod_j  \int dw_j  ~\pi_n(w_j) 
	\right)
	\\
	&= \begin{cases}
	\int_{\mathcal{C}_n} d\lambda^2 ~\pi_n(\lambda^2) & i=\mathcal{I}^{(\lambda^2)}_n \\
	\int_{\mathcal{C}_n} dw ~\pi_n(w) & i\in \mathcal{I}^{(w)}_n \\
	\int_{\mathcal{R}^+} d\lambda^2 ~\pi_n(\lambda^2) \int_{\mathcal{C}_n} dc_i ~\pi_n(c_i \mid \lambda^2) & i\in \mathcal{I}^{(c)}_n \\
	\end{cases}\label{eq:integral_cases}
	\end{align}
	In Equation \ref{eq:apply_fubini} we apply Fubini's theorem, which allows us to write a multiple integral as an interated integral. It is understood that the $i$th integral is over the restricted parameters space $[-C_n, C_n]$ while the remaining integrals are over the entire parameter space, meaning they integrate to 1. This allows us to write the result in Equation \ref{eq:integral_cases}.
	
	Therefore, by Equations \ref{eq:union} and \ref{eq:integral_cases} we have:
	\begin{align}
	\pi_n\left( \mathcal{F}_n^c \right) &\le 
	\underbrace{\int_{\mathcal{C}_n} d\lambda^2 ~\pi_n(\lambda^2)}_{\lambda^2 \text{ term}}
	+ \underbrace{\sum_{i \in \mathcal{I}_n^{(w)}} \int_{\mathcal{C}_n} dw ~\pi_n(w)}_{W \text{ term}} 
	+ \underbrace{\sum_{i \in \mathcal{I}_n^{(c)}} \int_{\mathcal{R}^+} d\lambda^2 ~\pi_n(\lambda^2) \int_{\mathcal{C}_n} dc_i ~\pi_n(c_i \mid \lambda^2)}_{C \text{ term}}
	\end{align}
	
	\item
	\subparagraph{Bound each term in the sum.}
	
	We will deal with each of these terms separately. 
	\begin{itemize}
		\item \textit{$W$ term.} With some minor difference for the dependence of the number of weight parameters on the network width ($DK_n$ in our case compared to $(D+2)K_n + 1$), equations 119-128 in \citep{lee_consistency_2000} show for all $n\ge N_w$ for some $N_w$:
		$$
		\sum_{i \in \mathcal{I}_n^{(w)}} \int_{\mathcal{C}_i(n)} \pi_n(w_i) ~dw_i \le \exp(-nr)
		$$
		
		\item \textit{$C$ term.} Since the parameter bound $C_n\to\infty$ as $n\to\infty$ and since the prior over the center parameters is defined over a bounded region, as $n\to\infty$ the bounded region will be contained in $[-C_n, C_n]$ and thus disjoint from $\mathcal{C}_i(n) := \Theta_i \setminus [-C_n, C_n]$. Thus, for all $n$ greater than some $N_c$, $\int_{\mathcal{C}_i(n)} \pi_n(c_i) ~dc_i=0$ for all center parameters. 
		
		\item \textit{$\lambda^2$ term}. 
		\begin{align}
		\int \pi_n(\lambda^2) d\lambda &= \int_{C_n}^\infty \frac{{\beta_\lambda}^{\alpha_\lambda}}{\Gamma(\alpha_\lambda)} \lambda^{2(\alpha_\lambda-1)} \exp(-\beta_\lambda^2 \lambda) d\lambda^2
		\\
		&\le \left(\frac{\beta_\lambda C_n}{\alpha_\lambda} \right)^\alpha_\lambda \exp( \alpha_\lambda - \beta_\lambda C_n) 
		&\pushright{\text{Chernoff Bound}} 
		\\
		&\le  \left(\frac{\beta_\lambda e}{\alpha_\lambda} \right)^\alpha \exp(\alpha_\lambda n^{b-a}) \exp(-\beta_\lambda \exp( n^{b-a}))
		&\pushright{C_n \le n^{b-a}} 
		\end{align}
		Taking the negative log we have:
		\begin{align}
		-\log\left( \int \pi_n(\lambda^2) d\lambda^2 \right) &\ge
		\underbrace{-\alpha\log\left(\frac{\beta e}{\alpha} \right)}_{:=A} + \beta \exp(n^{b-a}) - \alpha n^{b-a}
		\\
		&=A + \beta \left(\sum_{j=0}^\infty \frac{(n^{b-a})^j}{j!} \right)	- \alpha n^{b-a}
		\\
		&=A + \beta \left(1 +  n^{b-a} + \frac{1}{2}n^{2(b-a)} + \sum_{j=3}^\infty \frac{(n^{b-a})^j}{j!} \right)	- \alpha n^{b-a}
		\\
		&=\underbrace{(A+\beta) + (\beta-\alpha) n^{b-a} + \frac{1}{2}\beta n^{2(b-a)}}_{:= h(n)} + \beta \sum_{j=3}^\infty \frac{(n^{b-a})^j}{j!} 
		\\
		&= h(n) + \beta \sum_{j=3}^\infty \frac{(n^{b-a})^j}{j!}.
		\end{align}
		Now pick $k^* \in \{3,4,\dotsc \}$ such that $(b-a)k^* \ge 1$, so $n^{(b-a)k^*} \ge n$, and pick any $r\in (0, \beta / (k^*!))$. Then, since every term in the sum is positive, we have:
		\begin{align}
		-\log\left( \int \pi_n(\lambda^2) d\lambda^2 \right) &\ge h(n) + \beta \frac{n^{(b-a)k^*}}{k^*!}
		\\
		&\ge h(n) + \frac{\beta}{k^*!} n
		\\
		&\ge h(n) + r n
		\\
		&\ge r n &\pushright{\forall n\ge N_\lambda},
		\end{align}
		where the last inequality holds because $\beta>0$ and $(b-a)\in(0,1)$ clearly implies there exists an $N_\lambda>0$ such that for all $n\ge N_\lambda$, $h(n)>0$. Negating and exponentiating each side we have:
		\begin{align}
		\int \pi_n(\lambda^2) d\lambda^2 &\le \exp(-r n)  & \pushright{\forall n\ge N_\lambda}.
		\end{align}

	\end{itemize}
	
	\item
	\subparagraph{Bound sum.}
	For any $n\ge N_c$, since the C term is zero in this case, we have:
	\begin{align}
	\pi_n\left( \mathcal{F}_n^c \right) 
	&\le \sum_{i \in \mathcal{I}_n^{(w)}} \int_{\mathcal{C}_i(n)} \pi_n(w_i) ~dw_i
	+  \int_{\mathcal{C}_i(n)} \pi_n(\lambda^2) ~d\lambda^2
	\\
	&\le \exp(-r n) & \pushright{\forall n \ge N}
	\end{align}
	where the last inequality follows from Lemma \ref{lemma:bound_sum_of_exponentials} applied to the sequences:
	\begin{align}
	a_n &:= \sum_{i \in \mathcal{I}_n^{(w)}} \int_{\mathcal{C}_i(n)} \pi_n(w_i) ~dw_i \\
	b_n &:=  \int_{\mathcal{C}_i(n)} \pi_n(\lambda^2) ~d\lambda^2
	\end{align}
	which we already showed to be exponentially bounded above for large $n$. 
	
	\item
	\paragraph{Condition (ii)}
	
	Let $\gamma, \nu>0$.
	\item
	\subparagraph{Assume true regression function.}
	Assume $g_0 \in L_2$ is the true regression function
	
	\item
	\subparagraph{Find RBFN near ground truth function.}
	Set $\epsilon = \sqrt{\gamma/2}$. By Theorem \ref{thm:park_universal_1991} there exists an RBFN $g$ such that $\|g - g_0 \|_2 \le \epsilon$. We assume the center parameters of $g$ are contained in the bounded region $\mathcal{C}$ over which the Poisson process is defined, which can be made arbitrarily large.  
	
	\item
	\subparagraph{Define $M_\delta$.}
	
	Set $\delta = \epsilon / (3 n^a)$ and let $M_\delta$ be defined as in Lemma \ref{lemma:6}. Then by Lemma \ref{lemma:6}, for any $\tilde{g} \in M_\delta$ we have:
	\begin{equation}
	\sup_{x\in\mathcal{X}} \left(\tilde{g}(x) - g(x) \right)^2
	\le \left( 3 \tilde{K}_n \delta \right)^2 = \epsilon^2
	\end{equation}
	
	Next we show that $M_\delta \subset K_\gamma$ for all $\gamma>0$ and appropriately chosen $\delta$. This means we only need to show $\pi_n\left(M_\delta \right) \ge \exp(-n\nu)$, since $M_\delta \subset K_\gamma$ implies $\pi_n\left(K_\gamma \right) \ge \pi_n\left(M_\delta \right)$.
	
	\item
	\subparagraph{Show $M_\delta$ contained in $K_\gamma$.}
	Next we show that for any $\tilde{g}\in M_\delta$, $D_K(f_0, \tilde{f}) \le \gamma$ i.e. $M_\delta \subset K_\gamma$. The following are exactly equations 129-132 and then 147-151 from \cite{lee_consistency_2000}.
	\begin{align}
	D_K(f_0, \tilde{f}) &= \int \int f_0(x,y) \log \frac{f_0(x,y)}{\tilde{f}(x,y)}~dy~dx
	\\
	&= \frac{1}{2} \int\int \left[(y-\tilde{g}(x))^2 - (y-g_0(x))^2 \right] f_0(y\mid x) f_0(x) ~dy~dx
	\\
	&= \frac{1}{2} \int\int \left[-2y\tilde{g}(x) + \tilde{g}(x)^2 + 2yg_0(x) - g_0(x)^2 \right] f_0(y\mid x) f_0(x) ~dy~dx
	\\
	&= \frac{1}{2} \int (\tilde{g}(x) - g_0(x))^2 f_0(x)~dx
	\\
	&= \frac{1}{2} \int (\tilde{g}(x) - g(x) + g(x) - g_0(x))^2 f_0(x)~dx
	\\
	&\le \frac{1}{2}\left[
	\int 
	\underbrace{\sup_{x\in\mathcal{X}} (\tilde{g}(x) - g(x))^2 }_{\text{Lemma } \ref{lemma:6}}
	f_0(x)~dx
	+\int 
	\underbrace{(g(x) - g_0(x))^2}_{\text{Theorem } \ref{thm:park_universal_1991}}
	f_0(x)~dx \right. \\
	&\quad\quad \left. +2 
	\underbrace{\sup_{x\in\mathcal{X}} |\tilde{g}(x) - g(x)|}_{\text{Lemma } \ref{lemma:6}}
	\int  
	\underbrace{|g(x) - g_0(x)| }_{\text{Theorem } \ref{thm:park_universal_1991}}
	f_0(x) dx
	\right]
	\\
	&< \frac{1}{2}[\epsilon^2 + \epsilon^2 + 2\epsilon^2] 
	\\
	&=2\epsilon^2 = \gamma
	\end{align}
	
	\item
	\subparagraph{Show mass on $M_\delta$ is greater than exponential}
	\begin{align*}
	\pi_n\left(M_\delta \right) &= 
	\int_{\theta_1-\delta}^{\theta_1+\delta} \dotsc
	\int_{\theta_{\tilde{I}_n}-\delta}^{\theta_{\tilde{I}_n}+\delta} 
	\pi_n(\tilde{\theta}_1, \dotsc, \tilde{\theta}_{\tilde{I}_n}) 
	~d\tilde{\theta}_1 \dotsc d\theta_{\tilde{I}_n}
	\\
	&= 
	\int_{\theta_1-\delta}^{\theta_1+\delta} \dotsc
	\int_{\theta_{\tilde{I}_n}-\delta}^{\theta_{\tilde{I}_n}+\delta} 
	\pi_n(\tilde{\lambda^2}) \prod_i \pi_n(\tilde{c}_i \mid \tilde{\lambda^2} ) \prod_i \pi_n(w )
	~d\tilde{\theta}_1 \dotsc d\theta_{\tilde{I}_n} 
	\\
	&= 
	\int_{\lambda^2-\delta}^{\lambda^2+\delta} \pi_n(\tilde{\lambda^2}) \prod_{i=1}^{\tilde{I}_n^{(c)}} \int_{c_i-\delta}^{c_i+\delta} \pi_n(\tilde{c}_i \mid \tilde{\lambda^2}) ~d\tilde{c}_i~d\tilde{\lambda^2} 
	\times
	\prod_{i=1}^{\tilde{I}_n^{(w)}}  \int_{w_i-\delta}^{w_i+\delta} \pi_n(\tilde{w}_i) ~d\tilde{w}_i
	\\
	&= 
	\int_{\lambda^2-\delta}^{\lambda^2+\delta} \pi_n(\tilde{\lambda^2}) 
	\prod_{i=1}^{\tilde{I}_n^{(c)}} \int_{c_i-\delta}^{c_i+\delta} 
	\frac{1}{\mu(\mathcal{C})} 1_{[\tilde{c}_i \in \mathcal{C}]}
	~d\tilde{c}_i~d\tilde{\lambda^2} 
	\times
	\prod_{i=1}^{\tilde{I}_n^{(w)}}  \int_{w_i-\delta}^{w_i+\delta} \pi_n(\tilde{w}_i) ~d\tilde{w}_i
	\\
	&= 
	\underbrace{\int_{\lambda^2-\delta}^{\lambda^2+\delta} \pi_n(\tilde{\lambda^2}) ~d\tilde{\lambda^2}}_{\lambda^2 \text{ term}}
	\times
	\underbrace{\prod_{i=1}^{\tilde{I}_n^{(c)}} \int_{c_i-\delta}^{c_i+\delta} 
		\frac{1}{\mu(\mathcal{C})} 1_{[\tilde{c}_i \in \mathcal{C}]}
		~d\tilde{c}_i}_{\text{C term}}
	\times
	\underbrace{\prod_{i=1}^{\tilde{I}_n^{(w)}}  \int_{w_i-\delta}^{w_i+\delta} \pi_n(\tilde{w}_i) ~d\tilde{w}_i}_{W\text{ term}} 
	\end{align*}
	
	\begin{itemize}
		\item W term. The following correspond to equations 138-145 from \citep{lee_consistency_2000}.
		\begin{align}
		\text{W term} &= \prod_{i=1}^{\tilde{I}_m^{(w)}}  \int_{w_i-\delta}^{w_i+\delta} \pi_n(\tilde{w}_i) ~d\tilde{w}_i
		\\
		&=\prod_{i =1}^{\tilde{I}_n^{(w)}} \int_{w_i-\delta}^{w_i + \delta} (2\pi \sigma^2_w)^{-1/2} \exp\left(-\frac{1}{2\sigma^2_w} \tilde{w}_i^2 \right) ~d\tilde{w}_i			
		\\
		&\ge \prod_{i =1}^{\tilde{I}_n^{(w)}} 2\delta \inf_{\tilde{\theta}_i \in [\theta_i-1, \theta_i+1]} \left \{ (2\pi \sigma^2_w)^{-1/2} \exp\left(-\frac{1}{2\sigma^2_w} \tilde{w}_i^2 \right) \right\}
		\\
		&\ge \prod_{i =1}^{\tilde{I}_n^{(w)}} \delta \sqrt{\frac{2}{\pi \sigma^2_w}}  \exp\left(-\frac{1}{2\sigma^2_w} \zeta_i \right) & \pushright{\zeta_i:=\max\{(\theta_i-1)^2, (\theta_i+1)^2 \} }
		\\
		&\ge \left( \delta \sqrt{\frac{2}{\pi \sigma^2_w}} \right)^{\tilde{I}_n^{(w)}}  \exp\left(-\frac{1}{2\sigma^2_w} \zeta \tilde{I}_n^{(w)} \right) & \pushright{\zeta:=\max\{ \zeta_1,\dotsc,\zeta_{\tilde{I}_n^{(w)}} \} }
		\\
		&= \exp\left(- \tilde{I}_n^{(w)} \left[\delta^{-1} \sqrt{\frac{\pi \sigma^2_w}{2}}\right] \right)
		\exp\left(-\frac{1}{2\sigma^2_w} \zeta \tilde{I}_n^{(w)} \right)
		\\
		\\
		&= \exp\left(- \tilde{I}_n^{(w)} \left[\frac{3n^a}{\epsilon} \sqrt{\frac{\pi \sigma^2_w}{2}}\right] -\frac{1}{2\sigma^2_w} \zeta \tilde{I}_n^{(w)}\right)
		\\
		&= \exp\left(- \tilde{I}_n^{(w)} \left[a\log n - \log\sqrt{\frac{9\pi \sigma^2_w}{2\epsilon^2}} + \frac{1}{2\sigma^2_w} \zeta\right] \right)
		\\
		&= \exp\left(- \tilde{I}_n^{(w)} \left[2a\log n + \frac{1}{2\sigma^2_w} \zeta  \right]\right) &\pushright{\text{for large } n} \label{eq:w_large_n}
		\\
		&\ge \exp\left(- D n^a \left[2a\log n + \frac{1}{2\sigma^2_w} \zeta  \right]\right)
		&\pushright{\tilde{I}_n^{(w)} \le  D n^a}
		\\
		&\ge \exp(-\nu n) &\pushright{\text{for large } n} \label{eq:w_large_n_2}
		\end{align}
		Let $N_w$ denote the integer large enough so that Equations \ref{eq:w_large_n} and \ref{eq:w_large_n_2} hold for $\nu/3$.
		
		\item C term.
		
		\begin{align}
		\prod_{i=1}^{\tilde{I}_n^{(c)}} \int_{c_i-\delta}^{c_i+\delta} 
		\frac{1}{\mu(\mathcal{C})} 1_{[\tilde{c}_i \in \mathcal{C}]}
		~d\tilde{c}_i
		&\ge \prod_{i=1}^{\tilde{I}_n^{(c)}} \frac{\delta}{\mu(\mathcal{C})}
		\\
		&\ge \left( \frac{\delta}{\mu(\mathcal{C})} \right)^{\tilde{I}_n^{(c)}}
		\\
		&=\exp\left(-D n^a \log\left[\frac{\mu(\mathcal{C})}{\delta} \right] \right)
		\\
		&=\exp\left(-D n^a \log\left[\frac{3\mu(\mathcal{C}) n^a}{\epsilon} \right] \right)
		\\
		&=\exp\left(-D n^a \left[ a\log n - \log\left(\frac{3\mu(\mathcal{C})}{\epsilon} \right) \right] \right)
		\\
		&=\exp\left(-D n^a \left[ 2a\log n \right] \right) &\pushright{\text{for large } n} \label{eq:c_large_n}
		\\
		&=\exp\left(-2aD n^a \log n \right)
		\\
		&\ge \exp(-\nu n) &\pushright{\text{for large } n} \label{eq:c_large_n_2}
		\end{align}
		Let $N_c$ denote the integer large enough so that Equations \ref{eq:c_large_n} and \ref{eq:c_large_n_2} hold for $\nu/3$.
		
		\item $\lambda^2$ term.
		\begin{align}
		\int_{\lambda^2-\delta}^{\lambda^2+\delta} \pi_n(\tilde{\lambda^2}) ~d\tilde{\lambda^2}
		&= \int_{[\lambda^2-\delta,\lambda^2+\delta]\cap \mathbb{R}^+ } \frac{\beta^\alpha}{\Gamma(\alpha)} \tilde{\lambda^2}^{\alpha-1} \exp(-\beta \tilde{\lambda^2})~d\tilde{\lambda^2} 
		\\
		&\ge \delta \left(\inf_{\tilde{\lambda^2}\in [\lambda^2-\delta,\lambda^2+\delta]\cap \mathbb{R}^+}\left\{
		\frac{\beta^\alpha}{\Gamma(\alpha)} \tilde{\lambda^2}^{\alpha-1} \exp(-\beta \tilde{\lambda^2})
		\right\} \right) \label{eq:delta_not_2delta}
		\\
		&\ge \delta   \underbrace{\left(\inf_{\tilde{\lambda^2}\in [\lambda^2-1,\lambda^2+1]\cap \mathbb{R}^+}\left\{
			\frac{\beta^\alpha}{\Gamma(\alpha)} \tilde{\lambda^2}^{\alpha-1} \exp(-\beta \tilde{\lambda^2})
			\right\} \right)}_{:=A} & \pushright{\text{for large } n} \label{eq:delta_large_n}
		\\
		&= \delta A 
		\\
		&= \frac{A\epsilon}{3 n^a} 
		\\
		&\ge \exp(-\nu n) & \pushright{\text{for large n}} \label{eq:delta_large_n_2}
		\end{align}
		In Equation \ref{eq:delta_not_2delta} we note that the length of the interval $[\lambda^2-\delta,\lambda^2+\delta]\cap \mathbb{R}^+$ is at least $\delta$, since $\lambda^2 \in \mathbb{R}^+$. In Equation \ref{eq:delta_large_n} we note that $\delta<1$ for large $n$, allowing us to define the quantity $A$ that does not depend on $n$. Let $N_\lambda$ denote the integer large enough so that Equations \ref{eq:delta_large_n} and \ref{eq:delta_large_n_2} hold for $\nu/3$.
		
	\end{itemize}
	
	\item
	\subparagraph{Bound product}
	
	Set $N_2 = \max\{N_w, N_c, N_\lambda \}$. Then for all $n\ge N_2$:
	\begin{align*}
	\pi_n\left(M_\delta \right) &\ge \exp(-n \nu/3) \exp(-n \nu/3) \exp(-n \nu/3) \\
	&= \exp(-n \nu) 
	\end{align*}
	This shows condition (ii). Thus, the conditions of Theorem \ref{thm:lee_consistency_2000_sieves} are met, so the model is Hellinger consistent. By Corollary \ref{thm:lee_consistency_2000_corollary1} this gives asymptotic consistency.

\end{proof}

Now we consider the case where the number of hidden units $K$ of the network is a parameter of the model. Since the center parameters follow a Poisson process prior with intensity $\lambda$ over the region $\mathcal{C}$, then conditional on $\lambda$, $K$ follows a Poisson distribution with parameter $\mu(\mathcal{C})\lambda$, where $\mu$ is the measure of $\mathcal{C}$.  We again denote the estimated regression function by $\hat{g}_n$ with the understanding that the number of hidden units not fixed. 

\begin{theorem} (PoRB-NET consistency for homogeneous intensity).
	For any square integrable ground truth regression function $g_0$, $\hat{g}_n$ is asymptotically consistent for $g$ as $n\to\infty$, i.e.
	\begin{equation}
	\int (\hat{g}_n(x) - g_0(x))^2 dx \overset{p}{\to} 0.
	\end{equation}
\end{theorem}
\begin{proof}
	Since the number of hidden units follows a Poisson prior, the proof of this result is exactly as in Theorem 7 of \cite{lee_consistency_2000}. Their result relies on their Theorem 8, but we have adapted this result in Theorem \ref{thm:lee_consistency_2000_poisson} to our model and the remainder of the proof requires no additional assumptions regarding the model. Asymptotic consistency follows from Corollary \ref{thm:corollary2}.
\end{proof}

\newpage
\section{Appendix: Other Synthetic Results}\label{sec:results_app}

The following synthetic examples illustrate desirable aspects of the proposed approach. In particular, PoRB-NET allows to:
\begin{enumerate}
	\item easily specify lengthscale and signal variance information
	\item adapt the network architecture based on the data
	\item encode different degrees of uncertainty in out-of-sample regions
\end{enumerate}


\subsection{PoRB-NETs can express stationary priors}

For a BNN (top row) and PoRB-NET (bottom row), Figure \ref{fig:priors} shows for 10k prior function samples the functions themselves, a histogram of the upcrossings of the x-axis, and a plot of $\text{Cov}(x-k/2, x+k/2)$, for $k=0,1,2$, which is essentially a variogram (from left to right). Each network has $N(0, 1)$ priors on all weight and bias parameters, while the PoRB-NET has a $N(1,0)$ prior on the scale parameters and a uniform intensity over $[-10,10]$ for the center parameters. 

\begin{figure}[H]
	\centering
	\begin{subfigure}[t]{0.3\textwidth}
		\centering
		\includegraphics[height=1.2in]{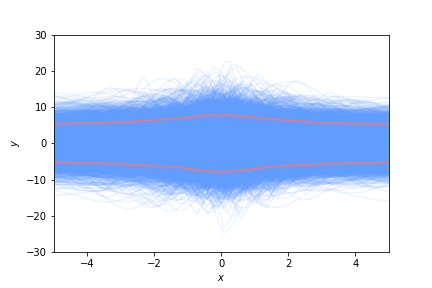}
	\end{subfigure}%
	~
	\begin{subfigure}[t]{0.3\textwidth}
		\centering
		\includegraphics[height=1.2in]{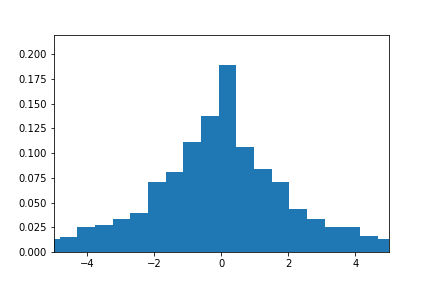}
	\end{subfigure}
	~
	\begin{subfigure}[t]{0.3\textwidth}
		\centering
		\includegraphics[height=1.2in]{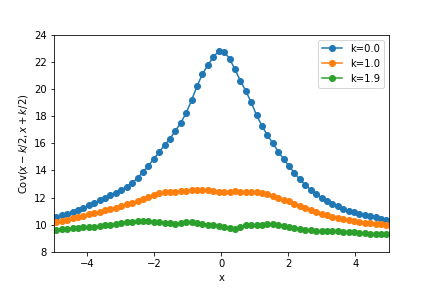}
	\end{subfigure}
	
	\begin{subfigure}[t]{0.3\textwidth}
		\centering
		\includegraphics[height=1.2in]{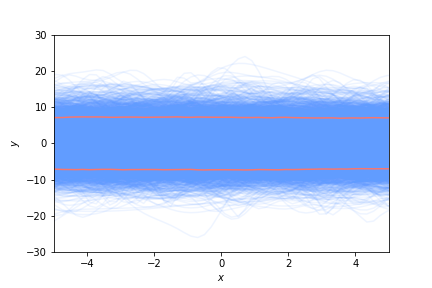}
		\caption*{Prior function samples}
	\end{subfigure}%
	~
	\begin{subfigure}[t]{0.3\textwidth}
		\centering
		\includegraphics[height=1.2in]{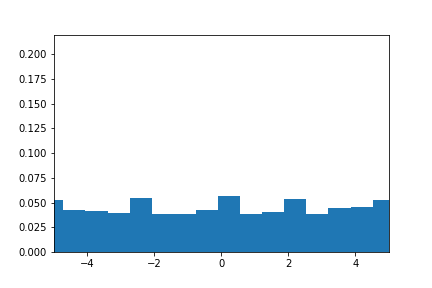}
		\caption*{Upcrossings of 0}
	\end{subfigure}
	~
	\begin{subfigure}[t]{0.3\textwidth}
		\centering
		\includegraphics[height=1.2in]{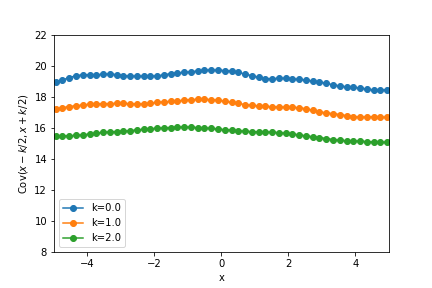}
		\caption*{Variogram}
	\end{subfigure}
	
	\caption{Priors for BNN (top row) vs. PoRB-NET (bottom row) \label{fig:priors}}
\end{figure}

Consequently, as shown in Figure \ref{fig:posterior}, modeling stationary functions away from the origin with a regular BNN (assuming the standard independent $w_k\sim\mathcal{N}(0,\sigma^2_w)$ and $b_k\sim\mathcal{N}(0,\sigma^2_b)$ priors on the weights and biases) requires making a tradeoff: Set the prior variance small, resulting in inability to capture the function and underestimated uncertainty away from the origin (left panel), or set the prior variance larger, resulting in nonstationary uncertainty (middle panel). Perhaps for some examples a ``sweet spot'' for this tradeoff exists, but a PoRB-NET is robust to this choice because the prior can easily be made stationary (right panel). For the PoRB-NET, we use a fixed uniform intensity defined over $[-5,5]$ and scaled so the prior expected network with is 12. Each BNN has 12 hidden units. 

\begin{figure}[H]
	\centering
	\begin{subfigure}[t]{0.3\textwidth}
		\centering
		\includegraphics[height=1.2in]{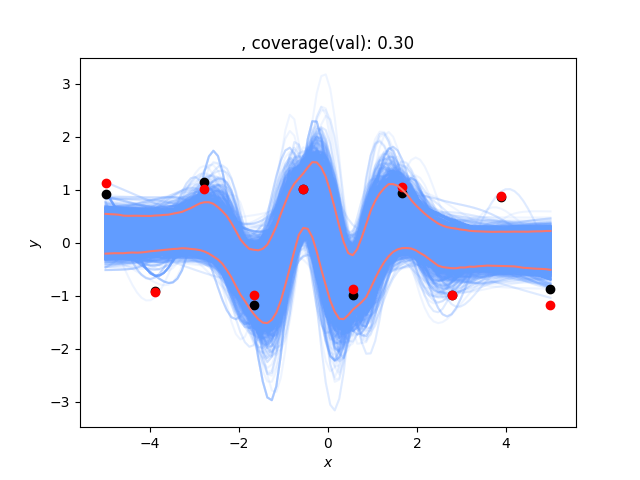}
		\caption{BNN with $\sigma^2_w=\sigma^2_b=1$.}
	\end{subfigure}%
	~
	\begin{subfigure}[t]{0.3\textwidth}
		\centering
		\includegraphics[height=1.2in]{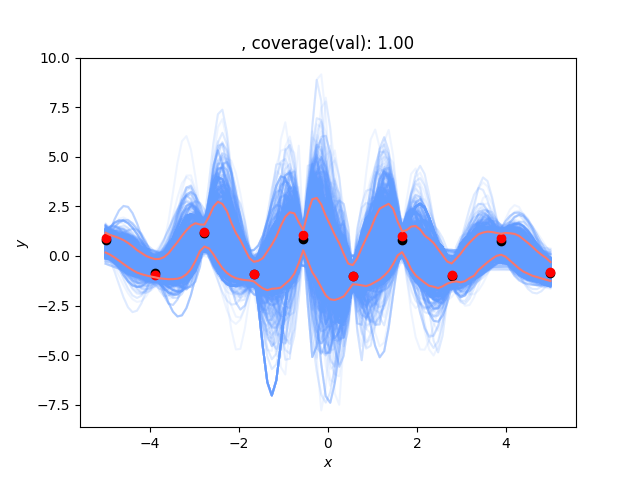}
		\caption{BNN with $\sigma^2_w=\sigma^2_b\approx 7.4$.}
	\end{subfigure}
	~
	\begin{subfigure}[t]{0.3\textwidth}
		\centering
		\includegraphics[height=1.2in]{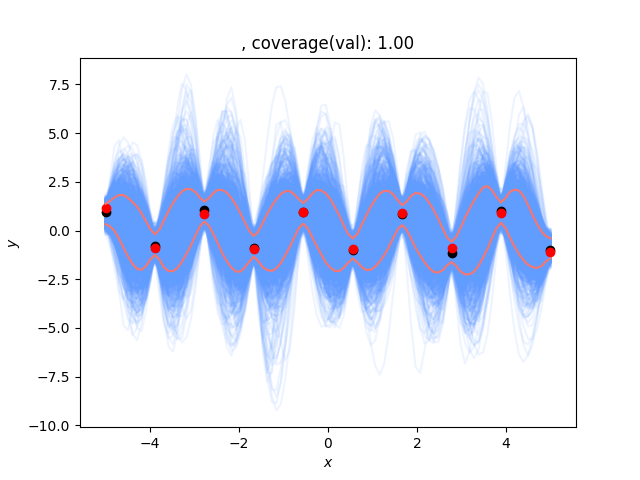}
		\caption{PoRB-NET.}
	\end{subfigure}
	\caption{For stationary functions and small amounts of data, a BNN faces a tradeoff between underestimated uncertainty away from the origin (left panel) versus overestimated uncertainty near the origin (right panel). A PoRB-NET can capture stationary functions (right panel). \label{fig:posterior}}
\end{figure}

\begin{figure}[H]
	\centering
	\begin{subfigure}[t]{0.3\textwidth}
		\centering
		\includegraphics[height=1.2in]{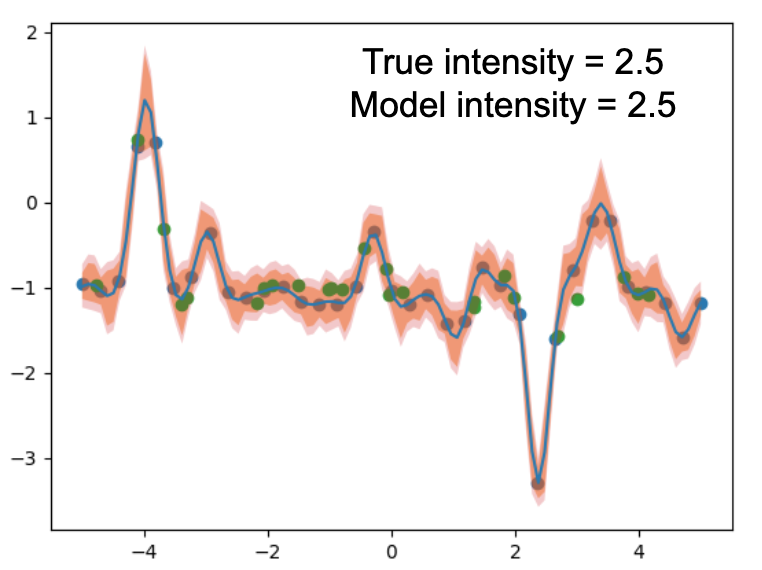}
		\caption{True LS low, Model LS low \\
			Test log likelihood: 11.84}
	\end{subfigure}%
	~
	\begin{subfigure}[t]{0.3\textwidth}
		\centering
		\includegraphics[height=1.2in]{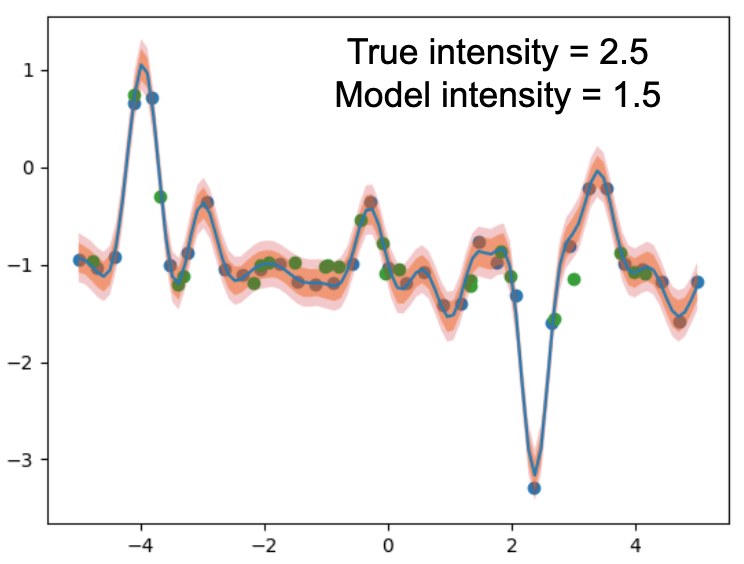}
		\caption{True LS low, Model LS high \\
			Test log likelihood: 10.97}	
	\end{subfigure}
	
	\begin{subfigure}[t]{0.3\textwidth}
		\centering
		\includegraphics[height=1.2in]{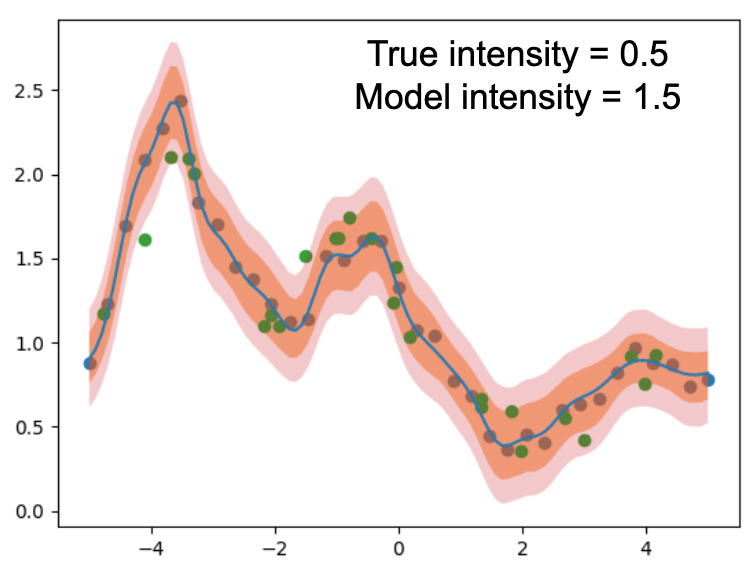}
		\caption{True LS high, Model LS low \\
			Test log likelihood: 8.27}
	\end{subfigure}%
	~
	\begin{subfigure}[t]{0.3\textwidth}
		\centering
		\includegraphics[height=1.2in]{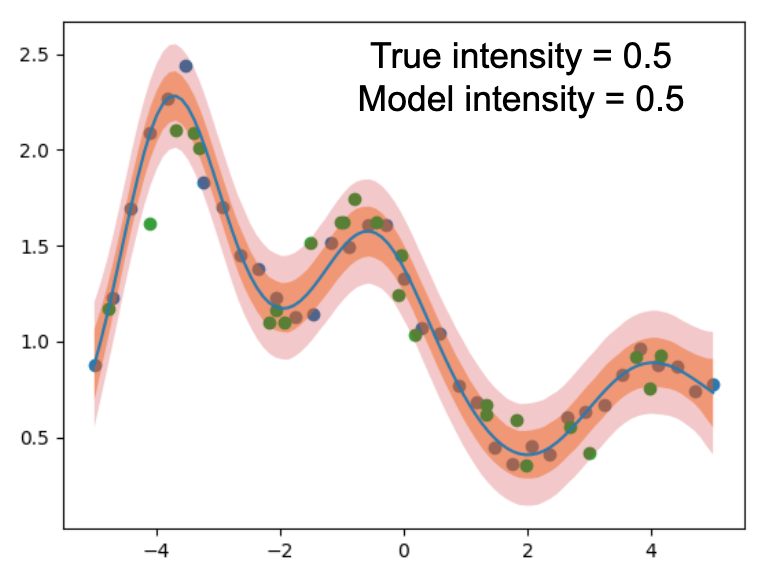}
		\caption{True LS high, Model LS high \\
			Test log likelihood: 10.51}
	\end{subfigure}
	
	\caption{Left-to-right: increased lengthscale (LS) for the PoRB-NET model. Top-to-bottom: increased lengthscale for the true function (drawn from a PoRB-NET prior). Notice that matching to the true intensity results in higher test log likelihood}. 
\end{figure}

\subsection{PoRB-NETs can adjust architecture based on data}

The posterior function samples in the left panel and middle panels, respectively, of Figure \ref{fig:toy1} show that a BNN with only 3 nodes has insufficient capacity to model this noisy sine wave, while a PoRB-NET, although initialized to and having a prior expectation of 3 nodes, is able to increase its capacity to between 5 and 7 nodes in response to the data. To isolate the impact of adaptive architecture rather than different prior specification, the intensity function of the PoRB-NET is 5 times larger inside $[-1,1]$ than elsewhere in $[-5,5]$, which yields a qualitatively similar prior distribution in function space to the BNN by concentrating the hidden unit centers near the origin. 

\begin{figure}[H]
	\centering
	\begin{subfigure}[t]{0.3\textwidth}
		\centering
		\includegraphics[height=1.2in]{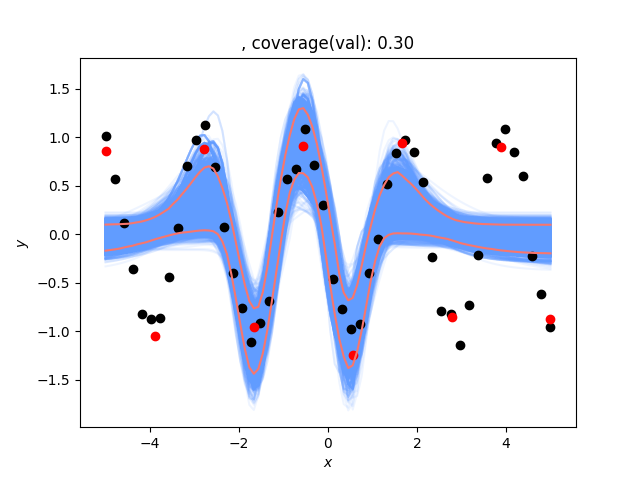}
		\caption{BNN: Posterior function samples.}
	\end{subfigure}%
	~
	\begin{subfigure}[t]{0.3\textwidth}
		\centering
		\includegraphics[height=1.2in]{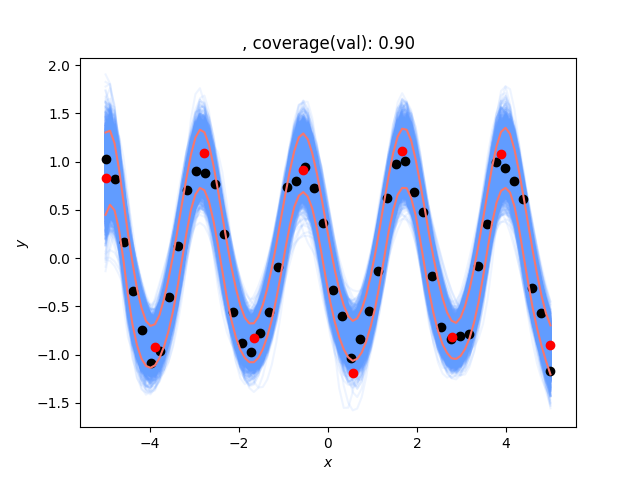}
		\caption{PoRB-NET: Posterior function samples.}
	\end{subfigure}
	~
	\begin{subfigure}[t]{0.3\textwidth}
		\centering
		\includegraphics[height=1.2in]{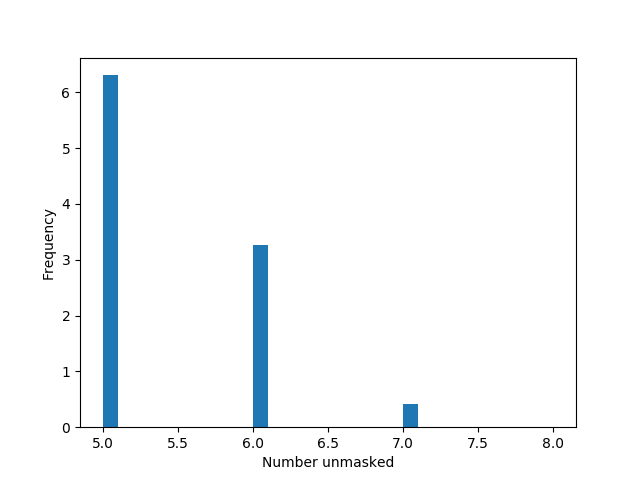}
		\caption{PoRB-NET: Posterior distribution of network width.}
	\end{subfigure}
	\caption{A BNN with only 3 nodes has insufficient capacity to model this data, while a PoRB-NET with 3 nodes in prior expectation is able to adjust its capacity in response to the data, settling on between 5 and 7 nodes.  \label{fig:toy1}}
\end{figure}

There are two perspectives on the advantages of an adaptive architecture. 
(i) If the data is more complex than expected in the prior, a PoRB-NET is somewhat robust to this choice, while a BNN will fail miserably. (ii) If the data is less complex than expected in the prior (suppose 12 nodes are expected for the example in Figure \ref{fig:toy1}), by intentionally choosing a prior that well underspecifies the expected capacity (3 nodes), the posterior will shrink towards a smaller architecture that can still model the data (5-7 nodes), while a BNN will stick to the larger architecture (12 nodes), leading to unnecessary computation and potential overfitting.

\subsection{PoRB-NETs can adjust the uncertainty in gaps in the training data}

By increasing the intensity in a gap in the training data, the lengthscale is reduced. Figure \ref{fig:gap} shows this for different piecewise constant intensity functions that are increased in the middle gap in the data. 

\begin{figure}[H]
	\centering
	\begin{subfigure}[t]{0.3\textwidth}
		\centering
		\includegraphics[height=1.2in]{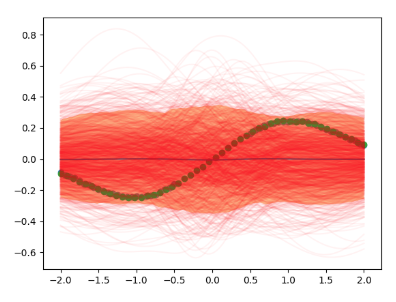}
	\end{subfigure}%
	~
	\begin{subfigure}[t]{0.3\textwidth}
		\centering
		\includegraphics[height=1.2in]{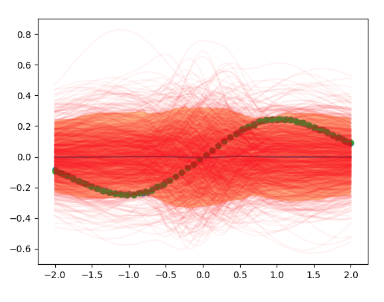}
	\end{subfigure}
	~
	\begin{subfigure}[t]{0.3\textwidth}
		\centering
		\includegraphics[height=1.2in]{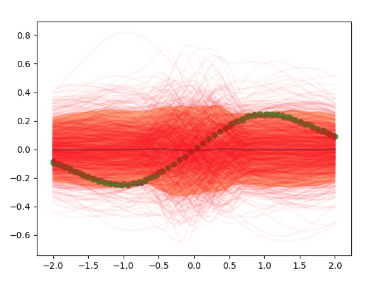}
	\end{subfigure}
	
	\begin{subfigure}[t]{0.3\textwidth}
		\centering
		\includegraphics[height=1.2in]{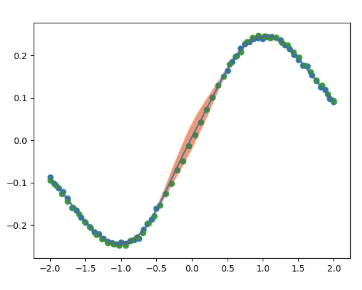}
		\caption*{2x intensity in gap}
	\end{subfigure}%
	~
	\begin{subfigure}[t]{0.3\textwidth}
		\centering
		\includegraphics[height=1.2in]{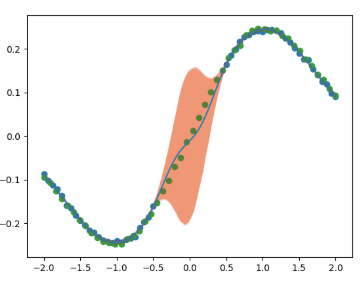}
		\caption*{3x intensity in gap}
	\end{subfigure}
	~
	\begin{subfigure}[t]{0.3\textwidth}
		\centering
		\includegraphics[height=1.2in]{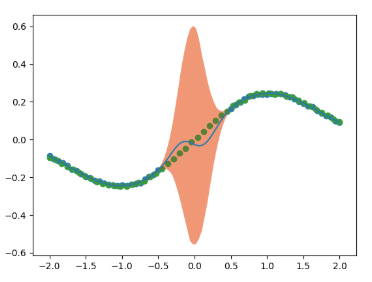}
		\caption*{4x intensity in gap}
	\end{subfigure}
	
	\caption{By adusting the Poisson process intensity a gap in the data (note that green points are test observations), the out of sample uncertainty can be adjusted. Higher intensity results in a smaller length scale.  \label{fig:gap}}
\end{figure}

\newpage

\section{Appendix: Other Results on Real Datasets}

\subsection{Inferred intensities}
Figure \ref{fig:intensity} shows the inferred intensities for the three real datasets discussed in the paper. Figure \ref{fig:realdb} is reproduced from main text for ease of comparison with Figure \ref{fig:intensity}.  The inferred intensities -- which represent an inverse lengthscale -- are larger near the most quickly-changing regions of the data. For example, in the motorcycle the function varies most in the middle. In the mimic dataset, the intensity is higher between the two spikes in the data. The lengthscale parameter of the GP in the inferred intensity is the parameter that adjusts the degree of smoothing between the spikes. For the finance dataset, the inferred intensity is higher near the regions of higher volatility, towards the beginning and end of the time series.

\begin{figure}[H]
	\includegraphics[width=1.\textwidth]{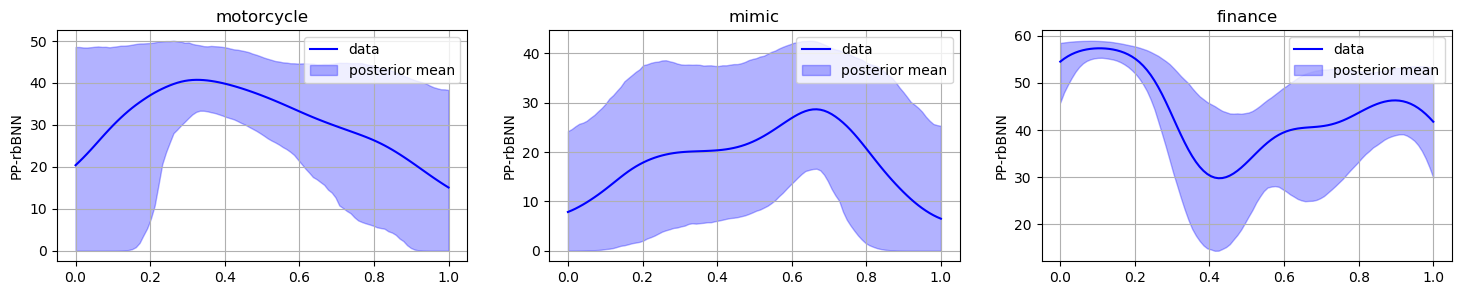}
	\caption{Inferred intensities for the PoRB-NET for the three real datasets. }\label{fig:intensity}
\end{figure}

\begin{figure}[H]
	\begin{center}
		\includegraphics[width=1.0\textwidth]{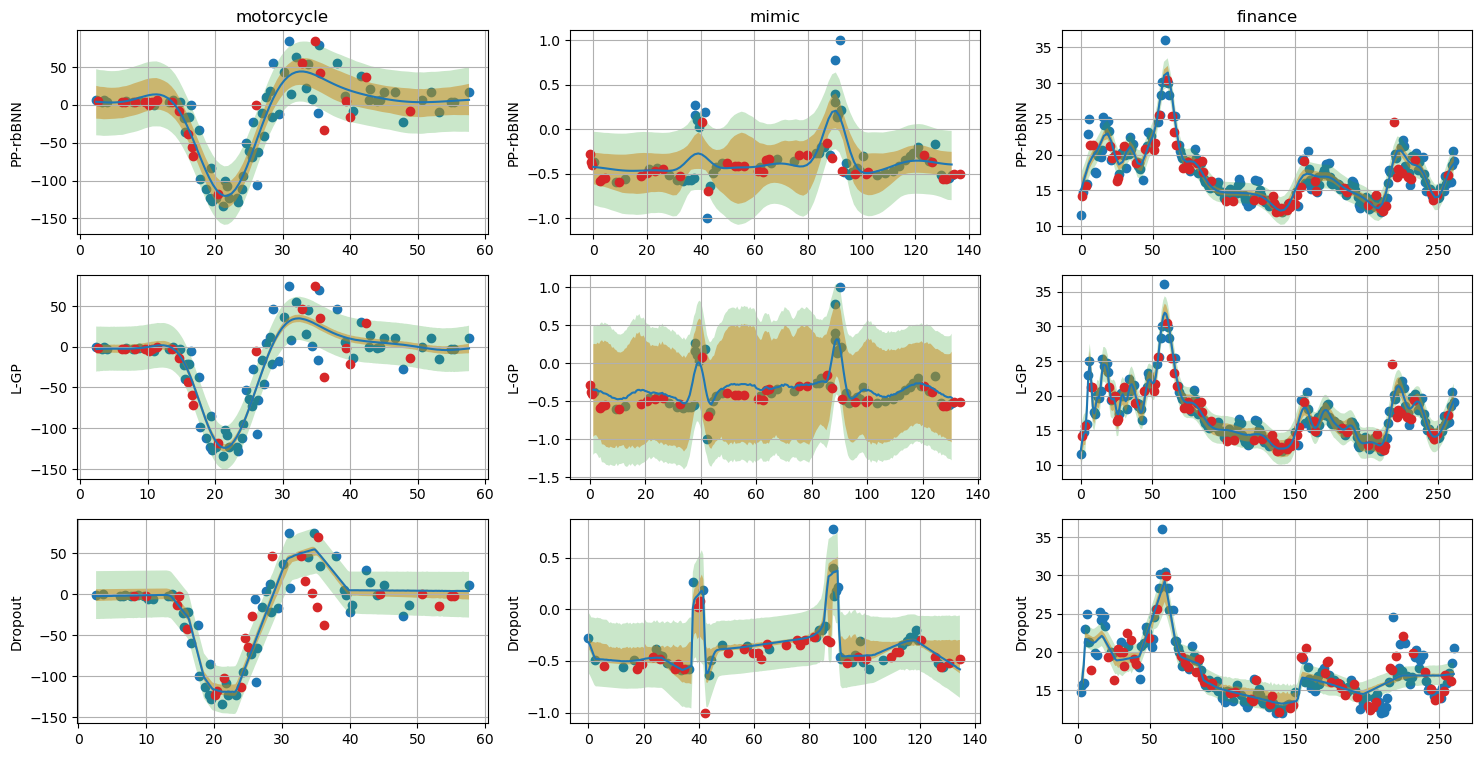}
		\caption{\textbf{PoRB-NET is able to capture non-stationary patterns in real scenarios.} Posterior predictive of PoRB-NET in three real datasets, in comparison to a GP with input-dependent length-scale (L-GP)~\cite{heinonen_non-stationary_2015}, and Dropout~\cite{gal_dropout_2015}. }\label{fig:realdb}
	\end{center}		
	\vspace{-0.4cm}
\end{figure}

We plot the posterior predictive densities for our model, and two other baselines: a GP with input-dependent length-scale (L-GP)~\cite{heinonen_non-stationary_2015}, 
and Dropout~\cite{gal_dropout_2015}.
PoRB-NET is able to capture complex non-stationary patterns in the data and yield uncertainties that mimic the behavior of L-GP. Interestingly, the learned intensity picks whenever the function exhibit faster variations (see the corresponding figure with the learned intensity function in the Appendix).

\section{Details on Experimental Setup}

\begin{itemize}
	\item Motorcycle dataset: Motorcycle accident data of Silverman (1985) tracks the acceleration force on the head of a motorcycle rider in the first moments after impact.
	\item CBE volatility index: downloaded from: \url{https://fred.stlouisfed.org/series/VIXCLS}
\end{itemize}

\subsection{Baselines}

\begin{itemize}
	\item \textbf{L-GP}: we use the code in Matlab from the authors that is publicly available at \url{https://github.com/markusheinonen/adaptivegp}.
	\item \textbf{Dropout}: we use the code in Tensorflow from the authors that is publicly available at \url{https://github.com/yaringal/DropoutUncertaintyExps}. We keep all default hyperparameters; we assume 1 single layer with 50 hidden units. Dropout rates were cross-validated using a grid search  $[0.0005,0.001,0.005,0.01]$.
\end{itemize}

\subsection{Simulation Setup}

In all the experiments, we use a random train-test split of 75-15. All datasets are normalized in a preprocessing step such that $x$ values fall in the range $[0,1]$, and $y$ values have zero mean and in the range $[-1,1]$.

We evaluate by computing the marginal test log likelihood as:
\begin{align}
\mathbb{E}_{p(\x^\star,\y^\star)}  \left[\log p(\y^\star|\x^\star,\mathcal{D})\right] =  \mathbb{E}_{p(\x^\star,\y^\star)} \left[\log \int p(\y^\star|\x^\star,\W) p(\W|\mathcal{D}) d\W \right]
\end{align}
\end{document}